\documentclass{article}

    \PassOptionsToPackage{numbers, compress}{natbib}

    \usepackage[preprint]{neurips_2018}

\usepackage[utf8]{inputenc} %
\usepackage[T1]{fontenc}    %
\usepackage{hyperref}       %
\usepackage{url}            %
\usepackage{booktabs}       %
\usepackage{amsfonts}       %
\usepackage{nicefrac}       %
\usepackage{microtype}      %
\usepackage{natbib}
\usepackage{url}  %
\usepackage{graphicx}  %
\usepackage{enumitem}
\usepackage{multirow}
\usepackage{multicol}
\usepackage{wrapfig}
\usepackage{xspace}

\usepackage{array,multirow}
\usepackage{bm}

\usepackage{subcaption}
\usepackage{amsmath}
\usepackage{amsfonts}
\usepackage{amssymb}
\usepackage{amsthm}
\usepackage{calc}
\usepackage[dvipsnames]{xcolor}
\usepackage{tikz} 
\usepackage[skins]{tcolorbox}
\usetikzlibrary{shapes,arrows}
\usetikzlibrary{shadows}
\usetikzlibrary{decorations.pathreplacing}

\tikzstyle{decision} = [diamond, draw, fill overzoom image=images/gradient_red,
    text width=4.5em, text badly centered, node distance=3cm, inner sep=0pt]
\tikzstyle{decision_b} = [diamond, draw, fill overzoom image=images/gradient_blue, 
    text width=4.5em, text badly centered, node distance=3cm, inner sep=0pt]
\tikzstyle{decision_c} = [diamond, draw, fill overzoom image=images/gradient, 
    text width=4.5em, text badly centered, node distance=3cm, inner sep=0pt]
\tikzstyle{block} = [rectangle, draw,
    text width=5em, text centered, rounded corners, minimum height=4em,fill overzoom image=images/gradient]
\tikzstyle{line} = [draw, -latex']
\tikzstyle{cloud} = [draw, ellipse,fill=red!20, node distance=3cm,
    minimum height=2em,text centered,text width=4.5em]
\tikzstyle{cloud_b} = [draw, ellipse,fill=blue!20, node distance=3cm,
    minimum height=2em,text centered,text width=4.5em]

\DeclareMathSizes{6.5}{6.5}{6.5}{6.5}

\DeclareMathOperator*{\argmin}{arg\,min}

\newcommand{\eodiff}[1]{\Delta_{\mathit{EO}_{#1}}}
\newcommand{\eopdiff}[1]{\Delta_{\mathit{EOp}_{#1}}}
\newcommand{\fairdiff}[1]{\Delta_{\mathit{Fair}_{#1}}}
\newcommand{\e}{\epsilon}
\newcommand{\symhyp}{\mathcal{H}\Delta\mathcal{H}}
\newcommand{\sensfeat}{A}
\newcommand{\sensset}{\mathcal{A}}
\newcommand{\distro}{\mathcal{D}}
\newcommand{\taskset}{\mathcal{Y}}
\newcommand{\tasklabel}{Y}
\newcommand{\emplabel}{\hat{Y}}
\newcommand{\sample}{Z}
\newcommand{\vcd}{d}
\newcommand{\dataone}{Dataset 1\xspace}
\newcommand{\datatwo}{Dataset 2\xspace}

\newtheorem{theorem}{Theorem}

\newtheorem{lemma}{Lemma}
\theoremstyle{definition}
\newtheorem{definition}{Definition}

\begin{document}

\title{Transfer of Machine Learning Fairness across Domains}
 
 \author{%
  Candice Schumann\thanks{Work done while at Google} \\
  University of Maryland\\
  \texttt{schumann@cs.umd.edu} \\
   \And
   Xuezhi Wang \\
   Google \\
   \texttt{xuezhiw@google.com} \\
   \AND
   Alex Beutel \\
   Google \\
   \texttt{alexbeutel@google.com} \\
   \And
   Jilin Chen \\
   Google \\
   \texttt{jilinc@google.com} \\
   \And
   Hai Qian \\
   Google \\
   \texttt{hqian@google.com} \\
   \And
   Ed H. Chi \\
   Google \\
   \texttt{edchi@google.com} \\
}

\date{}
\allowdisplaybreaks

\maketitle

\begin{abstract}
\emph{If our models are used in new or unexpected cases, do we know if they will make fair predictions?}
Previously, researchers developed ways to debias a model for a single problem domain. 
However, this is often not how models are trained and used in practice.
For example, labels and demographics (sensitive attributes) are often hard to observe, resulting in auxiliary or synthetic data to be used for training, and proxies of the sensitive attribute to be used for evaluation of fairness.
A model trained for one setting may be picked up and used in many others, particularly as is common with pre-training and cloud APIs.
Despite the pervasiveness of these complexities, remarkably little work in the fairness literature has theoretically examined these issues.

We frame all of these settings as domain adaptation problems: how can we use what we have learned in a source domain to debias in a new target domain, without directly debiasing on the target domain as if it is a completely new problem? 
We offer new theoretical guarantees of improving fairness across domains, and offer a modeling approach to transfer to data-sparse target domains.  We give empirical results validating the theory and showing that these modeling approaches can improve fairness metrics with less data. 
\end{abstract}

\section{Introduction}
Much of machine learning research, and especially machine learning fairness, focuses on optimizing a model for a single use case \cite{DBLP:journals/corr/abs-1803-02453,Beutel17:Data}.  However, the reality of machine learning applications is far more chaotic.  It is common for models to be used on multiple tasks, frequently different in a myriad of ways from the dataset that they were trained on, often coming at significant cost \cite{sculley2015hidden}.  This is especially concerning for machine learning fairness -- we want our models to obey strict fairness properties, but we may have far less data on how the models will actually be used.  How do we understand our fairness metrics in these more complex environments?

In traditional machine learning, domain adaptation techniques are used 
when the distribution of training and validation data does not match the target distribution that the model will ultimately be tested against.  Therefore, in this paper we ask: if the model is trained to be ``fair'' on one dataset, will it be ``fair'' over a different distribution of data?
Instead of starting again with this new dataset, can we use the knowledge gained during the original debiasing to more effectively debias in the new space? 

It turns out that this framing covers many important cases for machine learning fairness.  We will use, as a running example, the task of income prediction, where some decisions will be made based on the person's predicted income and we want the model to perform ``fairly'' over a sensitive attribute such as gender. We primarily follow the \emph{equality of opportunity} \cite{hardt2016equality} perspective where we are concerned with one group (broken down by gender or race) having worse accuracy than another.  In this setting, there are a myriad of fairness issues that arise that we find domain adaptation can shed light on:

\textbf{Lacking sensitive features for training:} There may be  few examples where we know the sensitive attribute. In these cases, a proxy of the sensitive attribute have been used \cite{DBLP:journals/corr/abs-1806-11212}, or researchers need very sample-efficient techniques \cite{DBLP:journals/corr/abs-1803-02453,Beutel17:Data}.  For distant proxies, researchers have asked how well fairness transfers across attributes \cite{Lan17:Discrimantory}.  Here the sensitive attribute differs in the source and target domains.

\textbf{Data is not representative of application:} Dataset augmentation, models offered as an API, or models used in multiple unanticipated settings, are all increasingly common design patterns.  Even for machine learning fairness, researchers often believe limited training data is a primary source of fairness issues \cite{chen2018my} and will employ dataset augmentation techniques to try to improve fairness \cite{dixon2018measuring}.  How can we best make use of auxiliary data during training and evaluation when it differs in distribution from the real application?
		
\textbf{Multiple tasks:} In some cases having accurate labels for model training is difficult and instead proxy tasks with more labeled data are used to train the model, e.g., using pre-trained image or text models or using income brackets as a proxy for defaulting on a loan.  Again we ask: when does satisfying a fairness property on the original task help satisfy that same property on the new task?

Each of these cases are common throughout machine learning but present challenges for fairness. 
In this work, we explore mapping domain adaptation principles to machine learning fairness.  In particular, we offer the following contributions:
\vspace{-0.1in}
\begin{enumerate}
	\item \textbf{Theoretical Bounds:} We provide theoretical bounds on transferring equality of opportunity and equality of odds metrics across domains. Perhaps more importantly, we discuss insights gained from these bounds.
	\item \textbf{Modeling for Fairness Transfer:} We offer a general, theoretically-backed modeling objective that enables transferring fairness across domains.
	\item \textbf{Empirical validation:} We demonstrate when transferring machine learning fairness works successfully, and when it does not, through both synthetic and realistic experiments.
\end{enumerate}

\section{Related Work}
This work lies at the intersection of traditional domain adaptation and recent work on ML fairness.
\vspace{-0.1in}
\paragraph{Domain Adaptation}
Both \citet{Pan10:Survey}, and \citet{Weiss16:Survey} provide a survey on current work in transfer learning. 
One case of transfer learning is domain adaptation, where the task remains the same, but the distribution of features that the model is trained on (the source domain) does not match the distribution that the model is tested against (the target domain).
\citet{Ben07:Analysis} provide theoretical analysis of domain adaptation. \citet{BenDavid10:Theory} extend this analysis to provide a theoretical understanding of how much source and target data should be used to successfully transfer knowledge. \citet{Mansour09:Domain} provide theoretical bounds on domain adaptation using Rademacher Complexity analysis.  In later research, \citet{ganin2016domain} build on this theory to use an adversarial training procedure over latent representations to improve domain adaptation.
\vspace{-10pt}
\paragraph{Fairness in Machine Learning}
A large thread of recent research has studied how to optimize for fairness metrics during model training.
\citet{Li18:Towards} empirically show that adversarial learning helps preserve privacy over sensitive attributes. \citet{Beutel17:Data} focus on using adversarial learning to optimize different fairness metrics, and \citet{Madras18:Learning} provides a theoretical framework for understanding how adversarial learning optimizes these fairness goals. \citet{Zhang18:Mitigating} use adversarial training over logits rather than hidden representations. Other work has focused on constraint-based optimization of fairness objectives \cite{NIPS2016_6316,DBLP:journals/corr/abs-1803-02453}.
\citet{Tsipras18:There} however, provide a theoretical bound on the accuracy of adversarial robust models. They show that even with infinite data there will still be a trade-off of accuracy for robustness.
\citet{kallus2019assessing} look at fairness in personalization when sensitive attributes are missing. Similarly, \citet{chen2019fairness} look at measuring disparity when sensitive attributes are unknown.
\vspace{-10pt}
\paragraph{Domain Adaptation \& Fairness}
Despite the prevalence of using one model across multiple domains, in practice little work has studied domain adaptation and transfer learning of fairness metrics.
\citet{coston2019fair} look at domain adaptation for fairness where sensitive attribute labels are not available in both the source and target domains.
\citet{DBLP:conf/icml/KallusZ18} use covariate shift correction when computing fairness metrics to address bias in label collection.
More related, \citet{Madras18:Learning} show empirically that their method allows for fair transfer. The transfer learning here corresponds to preserving fairness for a single sensitive attribute but over different tasks. However, \citet{Lan17:Discrimantory} found empirically that fairness does not transfer well to a new domain. They found that as accuracy increased in the transfer process, fairness decreases in the new domain. It is concerning that these papers show opposing effects. Both of these papers offer empirical results on the UCI adult dataset, but neither provide a theoretical understanding of how and when fairness in one domain transfers to another.

 \section{Problem Formulation}

 We begin with some notation to make precise the problem formulation.
Building on our running example we have two domains: a source domain $\sample\sim\distro_S$, which is a feature distribution influenced by sensitive attribute $A_S\in\sensset_S$ (e.g., $\Pr_{Z\sim\distro_S}[Z|A_S=\textit{male}]\neq \Pr_{Z\sim\distro_S}[Z|A_S=\textit{female}]$), as well as a target domain $\distro_T$ influenced by sensitive attribute $A_T\in\sensset_T$ (e.g., $\Pr_{Z\sim\distro_T}[Z|A_T=\textit{black}]\neq \Pr_{Z\sim\distro_T}[Z|A_T=\textit{white}]$). In order for this to be a domain adaptation problem, we assume  
$\Pr_{Z\sim\distro_S}[Z|A_S]
\neq
\Pr_{Z\sim\distro_T}[Z|A_T]$.  Note, this can be true even if $\distro_S = \distro_T$ but the distributions conditioned on $A_S$ and $A_T$ differ.
We focus on binary classification tasks with label $\tasklabel\in\taskset$, e.g. income classification is shared over both domains.
For this task we can create a classifier by finding a hypothesis $g: \distro\rightarrow\taskset$ from a hypothesis space $\mathcal{H}$.

Let us assume that we can learn a ``fair'' classifier $g$ for the source domain and task. If we use a small amount of data from the target domain, will the fairness from the source sensitive attribute $A_S$ transfer to the target domain and sensitive attribute $A_T$? 
We can define the notion of a ``fairness'' distance -- how far away the classifier is from perfectly fair -- in a given domain $S$ as $\fairdiff{S}$. Within this formulation we consider two definitions of fairness. 

The first distance is \emph{equality of opportunity} \cite{hardt2016equality}. A classifier is said to be fair under equality of opportunity if the false positive rates (FPR) over sensitive attributes are equal. In other words if we have a binary sensitive attribute $\sensfeat$, then equality of opportunity requires that
$\Pr(\emplabel=1 | \sensfeat=0, \tasklabel=0) = \Pr(\emplabel=1 | \sensfeat=1, \tasklabel=0)$,
where $\emplabel$ gives the outcome of classifier $g$.
Thus, how far away a classifier $g$ is from equal opportunity (or the fairness distance of equal opportunity) can be defined as 
{\small
\begin{equation*}
  \eopdiff{S}(g) \triangleq \left| \mathbb{E}_{\sample_0^0\sim \distro_{S_0^0}}[g(\sample_0^0)] - \mathbb{E}_{\sample_1^0\sim \distro_{S_1^0}}[g(\sample_1^0)] \right|,  
\end{equation*}
}%
where $\distro_{S_\alpha^l}=P_{\sample \sim \distro_S}[\sample|\sensfeat=\alpha, \tasklabel=l]$. 
In our running example $\eopdiff{S}(g)$, where $A_S$ is gender, is the difference between the likelihood that a low-income man is predicted to be high-income and the likelihood that a low-income woman is predicted to be high-income. A symmetric definition and set of analysis can be made for false negative rate (FNR).

The second definition of fairness which we consider is \emph{equalized odds} \cite{hardt2016equality}. A classifier is said to be fair under equalized odds if both the FPR \emph{and} FNR over the sensitive attribute are equal:
Similar to equal opportunity, we define the fairness distance of equalized odds as:
{\small
\begin{align*}
&\eodiff{S}(g) \triangleq \left| \mathbb{E}_{\sample_0^0\sim \distro_{S_0^0}}[g(\sample_0^0)] - \mathbb{E}_{\sample_1^0\sim \distro_{S_1^0}}[g(\sample_1^0)] \right|  + \left| \mathbb{E}_{\sample_0^1\sim \distro_{S_0^1}}[1-g(\sample_0^1)] - \mathbb{E}_{\sample_1^1\sim \distro_{S_1^1}}[1-g(\sample_1^1)] \right|.
\end{align*}
}%

Again using our running example, the distance of equalized odds in the source domain is given by the difference of expected FPRs between females and males (as above), plus the difference of expected FNRs (high-income predicted to be low-income) between females and males.

Given a classifier $g$ that has a fairness guarantee in the source domain, the fairness distance in the target domain should be bounded by the fairness distance in the source domain:
{\small
\begin{equation}
\label{eq:fairness_epsilon}
\fairdiff{T}(g) \leq \fairdiff{S}(g) + \epsilon
\end{equation}
}
The key question we hope to answer is: what is $\epsilon$?

\section{Bounds on Fairness in the Target Domain}\label{sec:theory}

To expand inequality~\eqref{eq:fairness_epsilon} we need to start with some definitions.
Given a hypothesis space $\mathcal{H}$ and a true labeling function $f(\sample): \distro\rightarrow\taskset$, we can define the error of a hypothesis $g\in \mathcal{H}$ as $\e_S(g,f)=\mathbb{E}_{\sample \sim \distro_S}\left[|f(\sample)-g(\sample)|\right]$, the expectation of disagreement between the hypothesis $g$ and the true label $f$. We can then define the ideal joint hypothesis that minimizes the combined error over both the source and target domains as $g^*=\argmin_{g\in\mathcal{H}}\e_S(g,f)+\e_T(g,f)$.

Following \citet{BenDavid10:Theory} we define the $\mathcal{H}$-divergence between probability distributions as
{\small
\begin{equation}
d_\mathcal{H}(\distro,\distro') = 2\sup_{g\in \mathcal{H}} \left| \Pr\nolimits_\distro[I(g)] - \Pr\nolimits_{\distro'}[I(g)] \right|,
\end{equation}
}%
where $I(g)$ is the set for which $g\in \mathcal{H}$ is the characteristic function ($\sample\in I(g) \Leftrightarrow g(\sample)=1$). We can compute an approximation $\hat{d}_{\mathcal{H}}(\distro, \distro')$ by finding a hypothesis $h$ that finds the largest difference between the samples from $\distro$ and $\distro'$~\cite{Ben07:Analysis}. This divergence can be used to look at the differences in distributions, which is  important when moving from a source domain to a target domain.

Additionally, we defined the symmetric difference hypothesis space $\symhyp$ as the set of hypotheses
{\small
\begin{equation}
g\in\symhyp \iff g(\sample) = h(\sample) \oplus h'(\sample) \quad \text{for some } h,h'\in \mathcal{H},
\end{equation}
}
where $\oplus$ is the XOR function. The symmetric difference hypothesis space is used to find disagreements between a potential classifier $g$ and a true labeling function $f$.

\begin{theorem}\label{thm:eop_transfer_bound}
Let $\mathcal{H}$ be a hypothesis space of VC dimension $\vcd$. If $\ \mathcal{U}_{S_0^0},\ \mathcal{U}_{S_1^0},\ \mathcal{U}_{T_0^1},\  \mathcal{U}_{T_1^0}$ are samples of size $m'$, each drawn from $\distro_{S_0^0}$, $\distro_{S_1^0}$, $\distro_{T_0^0}$, and $\distro_{T_1^0}$ respectively, then for any $\delta\in(0,1)$, with probability at least $1-\delta$ (over the choice of samples), for every $g\in\mathcal{H}$ (where $\mathcal{H}$ is a symmetric hypothesis space) the distance from equal opportunity in the target space is bounded by
{\small
\begin{align*}
    \eopdiff{T}(g) \leq&\ \eopdiff{S}(g) + \frac{1}{2}\hat{d}_{\symhyp}(\mathcal{U}_{T_0^0},\mathcal{U}_{S_0^0}) + \frac{1}{2}\hat{d}_{\symhyp}(\mathcal{U}_{T_1^0},\mathcal{U}_{S_1^0}) \\
    &\ + 8 \sqrt{\frac{2\vcd\log(2m')+\log(\frac{2}{\delta})}{m'}} + \lambda_0^{0} + \lambda_1^0,
\end{align*}
}
where $\lambda_\alpha^l=\e_{S_\alpha^l}(g^*,f)+\e_{T_\alpha^l}(g^*,f)$.
\end{theorem}

Using both the definition of $\mathcal{H}$-divergence and symmetric difference hypothesis space, Theorem \ref{thm:eop_transfer_bound} provides a VC-dimension bound on the equal opportunity distance in the target domain given the equal opportunity distance in the source domain.  Due to space limitations, full proofs for all theorems can be found in Appendix~\ref{app:proofs}.

\begin{wrapfigure}[11]{r}{0.5\textwidth}
\vspace{-20pt}
\centering
\includegraphics[width=0.5\textwidth]{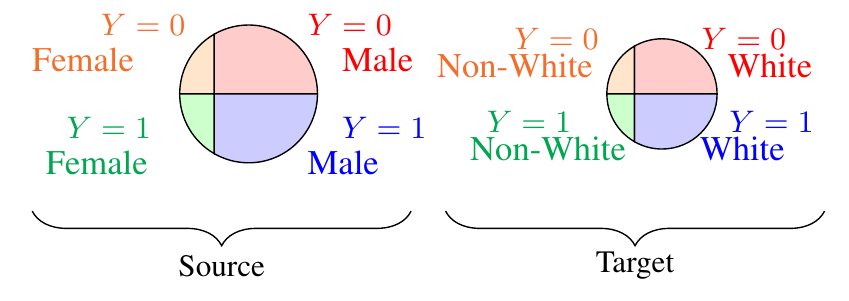}
\caption{Both the source and target distributions can be split into four quadrants: 1) negative minority 2) negative majority 3) positive minority and 4) positive majority.}
\label{fig:quadrants}
\vspace{-0.1in}
\end{wrapfigure}

This theorem provides insights on when domain adaptation for fairness can be used.
Firstly the $\hat{d}$ terms in the bound suggest that 1) the source and target distributions of negatively labeled items that have a sensitive attribute label of 0 should be close, and 2) the source and target distributions of the negatively labeled items that have a sensitive attribute label of 1 should be close. In Figure~\ref{fig:quadrants} the red quadrants should be close to the red quadrants while the orange quadrants should be close to the orange quadrants across domains. In traditional domain adaptation, ignoring fairness, the entire domains should be close (the entire circle), which means that if there are few minority data-points then the distance of the minority spaces will be ignored. The fairness bound instead puts equal emphasis on both the majority and minority.

Secondly, the $\lambda$ terms become small when the hypothesis space contains a function $g^*$ that has low error on both the source and target space on the two negative segments in each domain (the red and orange spaces in Figure~\ref{fig:quadrants}). 
Since we are looking at equal opportunity, the function $g^*$ only needs to have low error on the negative space for both the majority and minority. Therefore, we can use the trivial function $g^*(\sample)=0$ and the $\lambda$ terms go to 0. 

Lastly, Theorem~\ref{thm:eop_transfer_bound} depends on the VC-dimension $\vcd$. Since bounds with VC-dimensions explode with models like neural networks, we also provide bounds using Rademacher Complexity in Appendix~\ref{sec:radmacher}.

Equalized odds, while similar to equal opportunity, is a stricter fairness constraint. Theorem~\ref{thm:eo_transfer_bound} provides a VC-dimension bound on the difference of equal odds in the target domain given the source domain.

\begin{theorem}\label{thm:eo_transfer_bound}
Let $\mathcal{H}$ be a hypothesis space of VC dimension $\vcd$. If $\ \mathcal{U}_{S_\alpha^l}$ are samples of size $m'$, each drawn from $\distro_{S_\alpha^l}$ for all  $\alpha\in \sensset=\{0,1\}$ and  $l\in \taskset=\{0,1\}$, then for any $\delta\in(0,1)$, with probability at least $1-\delta$ (over the choice of samples), for every $g\in\mathcal{H}$ (where $\mathcal{H}$ is a symmetric hypothesis space) the distance from equalized odds in the target space is bounded by
{\small
\begin{align*}
    \eodiff{T}(g) \leq&\ \eodiff{S}(g) + \frac{1}{2}\hat{d}_{\symhyp}(\mathcal{U}_{T_0^0},\mathcal{U}_{S_0^0}) + \frac{1}{2}\hat{d}_{\symhyp}(\mathcal{U}_{T_1^0},\mathcal{U}_{S_1^0}) \\
    &\ + \frac{1}{2}\hat{d}_{\symhyp}(\mathcal{U}_{T_0^1},\mathcal{U}_{S_0^1}) + \frac{1}{2}\hat{d}_{\symhyp}(\mathcal{U}_{T_1^1},\mathcal{U}_{S_1^1})   + 16 \sqrt{\frac{2\vcd\log(2m')+\log(\frac{2}{\delta})}{m'}} + \lambda_\mathit{EO},
\end{align*}
}
where $\lambda_\mathit{EO} = \lambda_0^0 + \lambda_1^0 + \lambda_0^1 + \lambda_1^1$, and $\lambda_\alpha^l=\e_{S_\alpha^l}(g^*,f) + \e_{T_\alpha^l}(g^*,f)$.
\end{theorem}

The $\hat{d}_{\symhyp}$ terms suggest, that in order for equalized odds to transfer successfully then, 1) the source and target distributions of negatively labeled items on \textit{both} sensitive attribute labels 0 and 1 should be close, 2) the source and target distributions of the positively labeled items on \textit{both} sensitive attribute labels 0 and 1 should be close. In other words, all four quadrants of the source should individually be close to the respective four quadrants of the target in Figure~\ref{fig:quadrants}. 

Additionally, the $\lambda$ term shows that there should be a hypothesis that performs well over \emph{all} of these subspaces. This implication is intuitive given that equalized odds, by definition, wants a classifier to perform well in both the negative and positive space across both groups.

\section{Modeling to Transfer Fairness}
\label{sec:model}

\begin{wrapfigure}[13]{r}{0.4\textwidth}
\vspace{-60pt}
\centering
\includegraphics[width=0.4\textwidth]{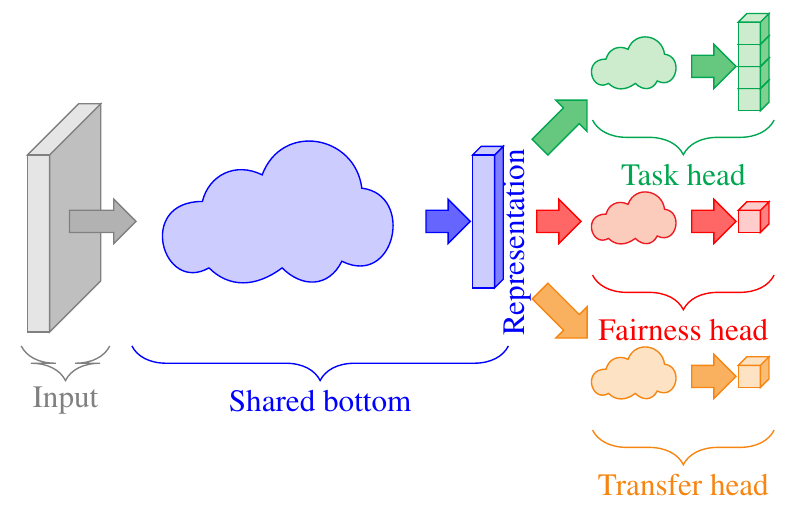}
\caption{At a high level, our general framework combines a primary training objective, a fairness objective, and a transfer objective to improve fairness goals in a target domain.  Table \ref{tab:terms} provides mathematical details for different configurations.
}
\label{fig:model_abstraction}
\end{wrapfigure}

With this theoretical understanding, how should we change our training?  As motivated previously, we consider the case where we have a small amount of labelled data (both labels $\taskset$ and sensitive attributes $\sensset$) in the target domain and a large amount of labelled data in the source domain.

As shown in the previous section, equality of opportunity will transfer \emph{if} the distance between the respective distributions of source and target are close together as visually portrayed in Figure \ref{fig:quadrants}.  
\citet{ganin2016domain} proved that traditional domain adaptation can be framed as minimizing the distance between source and target with adversarial training.  \cite{DBLP:journals/corr/LouizosSLWZ15,DBLP:journals/corr/EdwardsS15,Beutel17:Data,Li18:Towards} similarly have applied adversarial training to achieve fairness goals, and \citet{Madras18:Learning} proved that equality of odds can be optimized with adversarial training similar to domain adaptation.

We build on this intuition to design a learning objective for transferring equality of opportunity to a target domain.
Adversarial training conceptually enables minimizing a $\hat{d}$ term from Theorem \ref{thm:eop_transfer_bound}; and $\fairdiff{S}$ can be optimized using \cite{Beutel17:Data,Madras18:Learning} or one of the other myriad of traditional fairness learning objectives.  As such, we begin with the following loss:
{\small
\begin{align} 
    \min&\left[\sum_{\sample\sim(\distro_S\cup\distro_T)} L_Y(g(h(\sample)),f(\sample)) + \sum_{(A,\sample^0)\sim\distro_{S^0}} \lambda_{\textit{Fair}}L_A\left(a(h(\sample^0)),A\right) \right. \nonumber\\
    &\left.+\sum_{(d,\sample_0^0)\sim\left(\distro_{S_0^0}\cup\distro_{T_0^0}\right)}  \lambda_{\textit{DA}}L_d\left(d(h(\sample_0^0)),d\right) +\sum_{(d,\sample_1^0)\sim\left(\distro_{S_1^0}\cup\distro_{T_1^0}\right)}  \lambda_{\textit{DA}}L_d\left(d(h(\sample_1^0)),d\right) \right] \label{eq:adversarial_loss_function},
\end{align}
}%
where $L_Y(g(h(\sample)),f(\sample))$ is the loss function training $g(h(\sample))$ over hidden representation $h(\sample)$ to predict the task label $f(\sample)$.  To optimize $\fairdiff{S}$,  
$a(h(\sample^0))$ tries to predict the sensitive attribute $A$ from the source
and $L_A\left(a(h(\sample^0)),A\right)$ provides an adversarial loss that includes a negated gradient on $h$ following \cite{Beutel17:Data}.  For transfer, we minimize $\hat{d}$ terms by including another adversarial loss $L_d\left(d(h(\sample_l^\alpha)),d\right)$, where $d(h(\sample_l^\alpha))$ tries to predict whether a sample comes from the source or target domain. Each of these loss components maps to terms in Theorem~\ref{thm:eop_transfer_bound} as laid out in Table~\ref{tab:terms}.

\begin{table}
    \centering
    \begin{tabular}{|c|c|c|c|}
    \hline
        Loss Term & Theorem 1 & Adversarial  (Eq.~\ref{eq:adversarial_loss_function}) & Regularization (Eq.~\ref{eq:loss_fun}) \\
        \hline
        Fairness head & {\small$\eopdiff{S}(g)$} & {\small$\lambda_{\textit{Fair}}L_A\left(a(h(\sample^0)),A\right)$} & {\small$\lambda_{\textit{Fair}}L_{\mathit{MMD}}\left(a(h(\sample^0)),A\right)$} \\
        \hline
        \multirow{2}{*}{Transfer head} & {\small$\hat{d}_{\symhyp}(\mathcal{U}_{T_0^0},\mathcal{U}_{S_0^0})$} & {\small$\lambda_{\textit{DA}}L_d\left(d(h(\sample_0^0)),d\right)$} & \multirow{2}{*}{{\small$\lambda_{\textit{DA}}L_{\mathit{MMD}}\left(d(h(\sample^0)),d\right)$}} \\
        \cline{2-3}
         & {\small$\hat{d}_{\symhyp}(\mathcal{U}_{T_1^0},\mathcal{U}_{S_1^0})$} & {\small$\lambda_{\textit{DA}}L_d\left(d(h(\sample_1^0)),d\right)$} & \\
        \hline
    \end{tabular}
    \caption{Relationship between terms in Theorem~\ref{thm:eop_transfer_bound} and Loss functions }
    \label{tab:terms}
    \vspace{-0.2in}
\end{table}

Recently, \citet{Zhang18:Mitigating} used adversarial training on a one dimensional representation of the data (effectively the model's prediction).  From this perspective, we can use a wide variety of losses over predictions to replace adversarial losses, such as \cite{DBLP:conf/aistats/ZafarVGG17,beutel2019putting} minimizing the correlation between group and the one dimensional representation of the data.  Like previous work, we find that these approaches to be more stable and still effective in comparison to adversarial training, despite not being provably optimal.  In our experiments we use a MMD loss \cite{smola12,long15,DBLP:conf/nips/BousmalisTSKE16} over predictions:
{\small
\begin{align}
    \min & \left[ \sum_{Z\in \distro_S\cup \distro_T} L_Y(f(\sample),g(\sample)) + \sum_{(A,\sample^0)\sim\distro_{S^0}} \lambda_{\textit{Fair}}L_{\mathit{MMD}}\left(a(h(\sample^0)),A\right) \right. \nonumber \\
    & + \left. \sum_{(d,\sample^0)\sim\left(\distro_{S^0}\cup\distro_{T^0}\right)}  \lambda_{\textit{DA}}L_{\mathit{MMD}}\left(d(h(\sample^0)),d\right)\right] \label{eq:loss_fun},
\end{align}
}%
where $\lambda_{\textit{Fair}}L_{\mathit{MMD}}\left(a(h(\sample^0)),A\right)$ is the MMD regularization over the sensitive attributes in the source domain, $\lambda_{\textit{DA}}L_{\mathit{MMD}}\left(d(h(\sample^0)),d\right)$ is the MMD regularization over source/target membership. Again Table~\ref{tab:terms} maps the terms in Eq.~\ref{eq:loss_fun} to those in Theorem~\ref{thm:eop_transfer_bound}. 

Care must be taken when performing domain adaptation with regards to fairness. Either multiple transfer heads should be included in the loss for all necessary quadrants (See Figure~\ref{fig:quadrants} and Eq.~\ref{eq:adversarial_loss_function}), or balanced data -- equally representing all necessary quadrants -- should be used as in \cite{Madras18:Learning} and Eq.~\ref{eq:loss_fun}. 
Experiments in this paper use the MMD regularization as in Eq.~\ref{eq:loss_fun} and balanced data is used for both the fairness head as well as the transfer heads.

\vspace{-2mm}
\section{Experiments}
To better understand the theoretical results presented above, we now present both synthetic and realistic experiments exploring tightness of our theoretical bound as well as the ability to improve the transfer of fairness across domains during model training.

\vspace{-2mm}
\subsection{Synthetic Examples}
\begin{figure}[th]
    \centering
    \begin{subfigure}{0.18\columnwidth}
        \includegraphics[width=\textwidth]{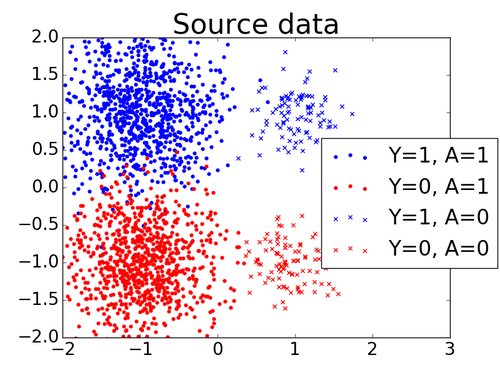}
        \caption{Source}
        \label{fig:source_data}
    \end{subfigure}
    \hspace{0.01\columnwidth}
    \begin{subfigure}{0.18\columnwidth}
        \includegraphics[width=\textwidth]{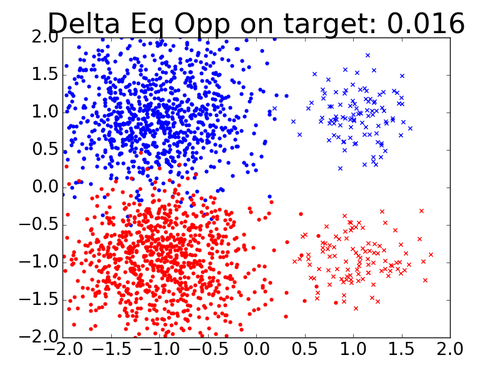}
       \caption{Target -1}
        \label{fig:center-10}
        \end{subfigure}
    \hspace{0.01\columnwidth}
    \begin{subfigure}{0.18\columnwidth}
        \includegraphics[width=\textwidth]{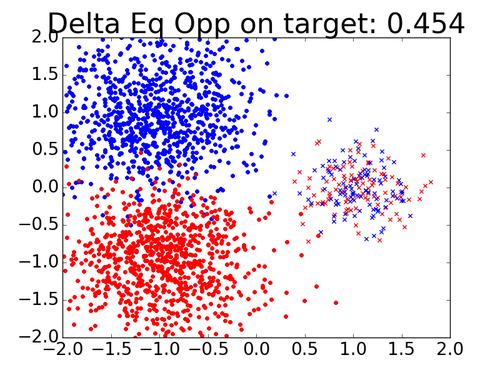}
        \caption{Target 0}
        \label{fig:center00}
    \end{subfigure}
    \hspace{0.01\columnwidth}
    \begin{subfigure}{0.18\columnwidth}
        \includegraphics[width=\textwidth]{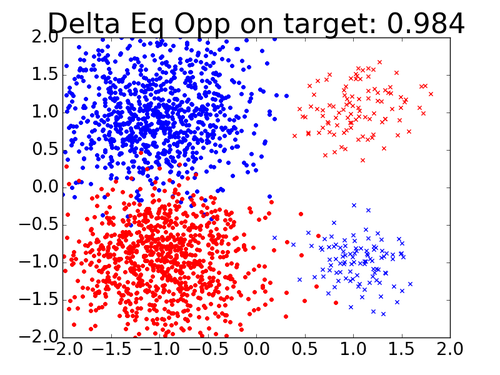}
        \caption{Target 1}
        \label{fig:center10}
    \end{subfigure}
    \hspace{0.01\columnwidth}
    \begin{subfigure}{0.18\columnwidth}
        \includegraphics[width=\textwidth]{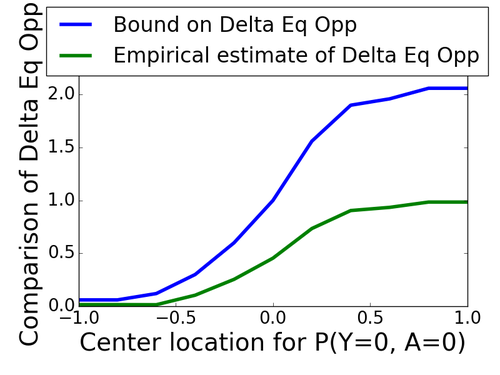}
        \caption{Target Fairness}
        \label{fig:compare_bound}
    \end{subfigure}
    \caption{Synthetic examples showing how distribution difference of $P(\sample|\tasklabel, \sensfeat=0)$ in the target domain affects theoretical and empirical equality of opportunity (best viewed in color). 
    In the title of each plot we give the equal opportunity distance $\eopdiff{T}(g)$ in the target domain.}
    \label{syn_example}
    \vspace{-0.1in}
\end{figure}

We show how well the theoretical bounds align with actual transfer of fairness.  
A synthetic dataset is used to examine how 
the distribution distance terms $\hat{d}_{\symhyp}(\mathcal{U}_{T_{\sensfeat=0}^{\tasklabel=0}},\mathcal{U}_{S_{\sensfeat=0}^{\tasklabel=0}})$ and $\hat{d}_{\symhyp}(\mathcal{U}_{T_{\sensfeat=1}^{\tasklabel=0}},\mathcal{U}_{S_{\sensfeat=1}^{\tasklabel=0}})$ in Eq.~\eqref{thm:eop_transfer_bound} affect the fairness distance of equal opportunity $\eopdiff{T}(g)$.

In this synthetic example, we generate data $\sample \in \mathbb{R}^2$ using Gaussian distributions.  As we can see in Figure \ref{fig:source_data}, the source domain consists of four Gaussians, with $\tasklabel=1$ largely lying above $\tasklabel=0$ and $\sensfeat=1$ lying to the left of $\sensfeat=0$; $\sensfeat=1$ is the majority of the data ($\sigma=0.5$ with $900$ samples). For $\sensfeat=0$, the data is generated using $\sigma=0.3$ with $100$ samples.
The target domain, like the source domain, consists of majority data with $\sensfeat=1$ and the data from $\sensfeat=1$ is generated from the same distribution in both domains: $\mathcal{U}_{T_{\sensfeat=1}^{\tasklabel=0}} \sim \mathcal{N}([-1,-1], \sigma)$ and $\mathcal{U}_{T_{\sensfeat=1}^{\tasklabel=1}} \sim \mathcal{N}([-1,1], \sigma)$.
However, in order to understand the transfer of fairness, we shift the distributions of $\mathcal{U}_{T_{\sensfeat=0}^{\tasklabel=0}} \sim \mathcal{N}([1,c],\sigma) $ and $\mathcal{U}_{T_{\sensfeat=0}^{\tasklabel=1}} \sim \mathcal{N}([1,-c],\sigma)$ in the target domain ($c=-1, 0, 1$ for \ref{fig:center-10}, \ref{fig:center00} and \ref{fig:center10}, respectively). 
By varying the overlap between these distributions, and their alignment with the source data, we are able to understand the relationship between the $\hat{d}_{\symhyp}$ terms above and the fairness distance of equal opportunity $\eopdiff{T}(g)$.
For each setting, we train linear classifiers on the source domain and examine the performance in the target domain.

\vspace{-10pt}
\paragraph{Qualitative Analysis}
We see in Fig.~\ref{fig:center-10} that when the distribution $P(\sample|\tasklabel=0, \sensfeat=0)$ across domains is close, thus a smaller $\hat{d}_{\symhyp}(\mathcal{U}_{T_0^0},\mathcal{U}_{S_0^0})$, there is better transfer of fairness the source to the target domain, seen in the smaller $\eopdiff{T}(g)$.
As the distribution distance gets larger, the $\eopdiff{T}(g)$ also increases.
Consider the worst case of a sign flip for the minority $\sensfeat=0$, as shown in Fig.~\ref{fig:center10}: the FPR for the majority $\sensfeat=1$ is close to $0\%$, while the FPR for the minority $\sensfeat=0$ is close to $100\%$. 

\vspace{-10pt}
\paragraph{Quantitative Analysis}
In Figure \ref{fig:compare_bound}, we compare the derived bound of $\eopdiff{T}(g)$ (Eq.~\ref{thm:eop_transfer_bound}) with its empirical estimate as we vary $c$\footnote{As in \cite{Ben07:Analysis},  $\hat{d}_{\symhyp}(\mathcal{U}_{T_0^0},\mathcal{U}_{S_0^0})$ is estimated by a linear classifier trained on samples $\mathcal{U}_{T_0^0},\mathcal{U}_{S_0^0}$. The plot omits the VC term for simplicity, which is relatively small when sample size $m'$ is large and VC-dimension $\vcd$ is low.}.
As shown in Figure~\ref{fig:compare_bound}, the theoretical bound on the equal opportunity distance is close to the observed equal opportunity distance when the distance between the negative minority space across domains, $\hat{d}(\mathcal{U}_{T_0^0}, \mathcal{U}_{S_0^0})$, is small. 
This suggests, minimizing the domain distance terms in Eq.~\ref{thm:eop_transfer_bound} could lead to a better equal opportunity transfer.

\vspace{-2mm}
\subsection{Real Data}
\begin{figure}
    \centering
    \begin{subfigure}{0.23\columnwidth}
        \includegraphics[width=\textwidth]{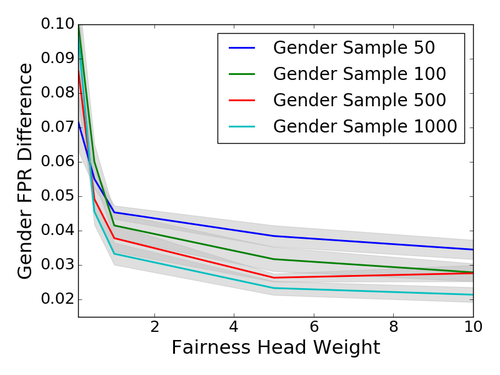}
        \caption{Effect of fairness head: Improving $\eopdiff{\rm gender}$ with varying number of gender-balanced samples.}
        \label{fig:fpr_diff_gender_gender_data_only}
    \end{subfigure}
    \hspace{0.01\columnwidth}
    \begin{subfigure}{0.23\columnwidth}
        \centering
        \includegraphics[width=\textwidth]{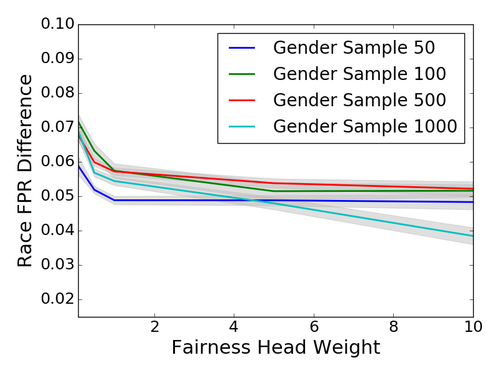}
        \caption{Some natural transfer occurring without explicit transfer: $\eopdiff{\rm race}$ is improved with gender data.}
        \label{fig:fpr_diff_race_gender_data_only}
    \end{subfigure}
    \hspace{0.01\columnwidth}
    \begin{subfigure}{0.23\columnwidth}
        \includegraphics[width=\textwidth]{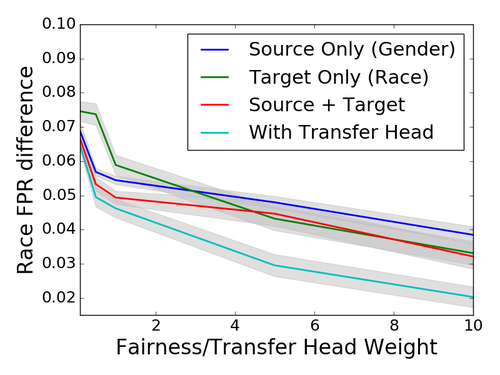}
        \caption{Effect of transfer head: better transfer from gender (1000 samples) to race (50 samples).}
        \label{fig:transfer_gender_to_race_50}
    \end{subfigure}
    \hspace{0.01\columnwidth}
    \begin{subfigure}{0.23\columnwidth}
        \includegraphics[width=\textwidth]{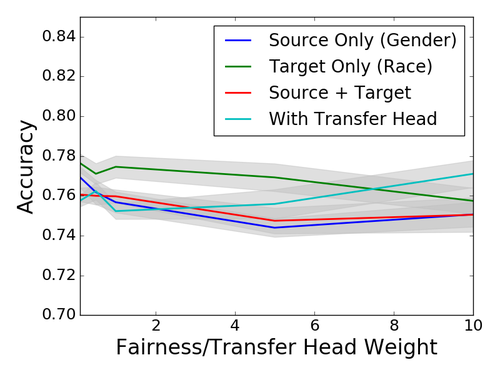}
        \caption{Accuracy graph for transferring from gender (1000 samples) to race (50 samples).}
    \end{subfigure}
    \caption{Effect of fairness/transfer head on the UCI data. The shaded areas show the standard error of the mean across trials. Note the head weight (x-axis) starts from $0.1$.}
    \vspace{-0.2in}
\end{figure}
We now explore how and when our proposed modeling approach in Section \ref{sec:model} facilitates the transfer of fairness from the source to the target domain on two real-world datasets.
Note, we use these datasets exclusively for understanding our theory and model, and \emph{not} as a comment on when or if the proposed tasks and their application are appropriate, as in \cite{DBLP:journals/corr/abs-1803-02453}.

\textbf{\dataone:} 
The UCI Adult\footnote{https://archive.ics.uci.edu/ml/datasets/adult} dataset contains census information of over 40,000 adults from the 1994 Census, with the task of determining income brackets of $>\!\$50,000$ or $\leq\!\$50,000$. We focus on two sensitive attributes: binary valued gender,  and race, converted to binary values [`white', `non-white'] as done by \citet{Madras18:Learning}.

\textbf{\datatwo:} 
As in \cite{DBLP:journals/corr/abs-1803-02453} we use 
ProPublica’s COMPAS recidivism data\footnote{https://github.com/propublica/compas-analysis} 
to try to predict recidivism for over 10,000 defendants based on age, gender, demographics, prior crime count, etc. We again focus on two sensitive attributes: gender and race (binarized to [`white', `non-white']).  
\vspace{-10pt}
\paragraph{Experiment Setup}
For both datasets, cross-validation is used to choose the hyper-parameters. Comparable baseline accuracy (around $84\%$ for \dataone and $80\%$ for \datatwo, see appendix~\ref{appendix_exp} for more details) is achieved with $64$ embedding dimension for categorical features, single hidden layer with $256$ shared hidden units, $512$ batch size, $0.1$ learning rate with Adagrad optimizer, and $10,000$ epochs for training.
We perform $30$ runs for each set of experiments and average over the results.
\vspace{-10pt}
\paragraph{Sparsity Issues and Natural Transfer}
We examine the effectiveness of just the fairness heads in the proposed model.  The amount of gender-balanced data created for the fairness head is varied to observe how applying the fairness head affects the FPR difference.

We examine how this procedure effects the FPR difference across genders (i.e., the FPR difference between ``Female'' and ``Male'' examples).
Figure \ref{fig:fpr_diff_gender_gender_data_only} shows that the fairness head works as expected: with sufficient data and a large enough weight, the fairness head is able to improve the FPR gap across genders.  Further, we find that with very few examples on which to apply the fairness head, the gender FPR gap does not close.  This aligns with previous results found in \cite{Beutel17:Data,Madras18:Learning,beutel2019putting}.

Second, we examine how running the fairness head on gender affects the FPR gap across race.  As shown in Figure \ref{fig:fpr_diff_race_gender_data_only}, there is a natural transfer of equal opportunity from gender to race -- applying a fairness loss with respect to gender also improves the fairness of the model with respect to race.  This highlights that sometimes there is a natural transfer of equal opportunity, presenting general value in improving the FPR gap with respect to gender, and no explicit transfer optimization is needed.  (Similar to the transfer questions posed previously by \citet{Madras18:Learning} and \citet{DBLP:journals/corr/abs-1806-11212}).

\vspace{-10pt}
\paragraph{Effectiveness of Transfer Head}
We now explore how adding the transfer head can further improve equality of opportunity in the target domain.
We compare four different model arrangements: 
(1) \textbf{Source Only}: We only add a fairness head for the source domain; (2) \textbf{Target Only}: We only add a fairness head for the target domain; (3) \textbf{Source+Target}: We add two fairness heads, one for source and for target; (4) \textbf{Transfer}: We include three heads -- both source and target fairness heads as well as the transfer head for equality of opportunity.

\textit{Experiment setting:} As in typical transfer learning setting, we will focus on the case where we observe a large number of samples in the source domain (e.g., 1000 for each race ``white'' and ``non-white''), but a smaller sample size in the target domain (e.g., 100 for each gender ``male'' and ``female''), and the same for gender to race.  We explore equality of opportunity with respect to FPR in the target domain, as we vary the weight on the fairness and transfer heads.

\textit{Results:} Figure~\ref{fig:transfer_gender_to_race_50} shows that including the transfer head results in a better equal opportunity transfer, compared to the same setting without transfer (Figure~\ref{fig:fpr_diff_race_gender_data_only}).
Table \ref{tab:comarison_table} summarizes the full results on both datasets. We can see that including both the fairness heads and the transfer head consistently gives the best improvement in equal opportunity (FPR difference) in almost all cases.

\vspace{-10pt}
\paragraph{Effect of Target Sample Size}
Last, we consider how the amount of data from the target domain affects our ability to improve equal opportunity there, as sample efficiency is a core challenge.

\textit{Experiment setting:} We follow a similar experimental procedure as before with two modifications.  First, we vary the number of samples we observe for each sensitive group in the target domain to be in $\{50,100, 500,1000\}$.  We examine the efficacy of the four approaches depending on the amount of data available for debiasing in the target domain.  Second, this analysis is performed for both transferring from race (source) to gender (target), as well as from gender (source) to race (target).

\textit{Results:} Table \ref{tab:comarison_table} summarizes the results.
Applying the fairness and transfer heads to the large amount of source data closes the FPR gap in the target domain.  
Increasing the amount of data in the target domain significantly helps the performance of the ``Target Only'' and the ``Source+Target'' models.  This is intuitive since directly debiasing in the target domain is feasible with sufficient data.  With sufficient data, the results converge to be approximately equivalent to the transfer model.  

These experiments show that the transfer model is effective in decreasing the FPR gap in the target domain and is more sample efficient than previous methods.

\begin{table}[tbh]
\scriptsize
    \centering
    \begin{tabular}{|c|c|c|c|c|c|c|}
    \hline
    & & & \multicolumn{4}{c|}{Smallest FPR difference achieved on Target (FPR-diff $\pm$ std. dev)} \\
    \hline
     & \shortstack{Source to \\Target} & \shortstack{\#Target\\ Samples} &
    Source only &  Target only &  Source + Target &  \shortstack{With Transfer\\ Head}\\
    \hline
         \multirow{8}{*}{\dataone} 
         & \multirow{4}{*}{\shortstack{Gender \\to\\ Race}} 
         & 50 & $0.038\pm0.013$ & $0.033\pm0.019$ & $0.032\pm0.020$ & $\bm{0.020\pm 0.016}$ \\ 
         \cline{3-7} & & 100 & $\bm{0.038\pm0.013}$ & $\bm{0.038\pm0.021}$ & $0.044\pm0.024$ & $0.040\pm0.024$\\  
         \cline{3-7} & & 500 & $0.038\pm0.013$ & $0.053\pm 0.010$ & $0.043\pm0.017$ & $\bm{0.025\pm0.018}$\\ 
         \cline{3-7} & & 1000 & $0.038\pm0.013$ & $\bm{0.027\pm0.018}$ & $\bm{0.027\pm0.019}$ & $0.031\pm0.021$\\
         \cline{2-7} & \multirow{4}{*}{\shortstack{Race \\to\\ Gender}} 
         & 50 &  $0.061\pm0.054$ & $0.035\pm0.015$ & $0.020\pm0.026$ & $\bm{0.008\pm0.009}$  \\ 
         \cline{3-7} & & 100 & $0.061\pm0.054$ & $0.028\pm0.014$ & $0.021\pm0.015$ & $\bm{0.009\pm0.011}$ \\  
         \cline{3-7} & & 500 & $0.061\pm0.054$ & $0.028\pm0.013$ & $0.019\pm0.013$ & $\bm{0.014\pm0.011}$ \\ 
         \cline{3-7} & & 1000 & $0.061\pm0.054$ & $0.021\pm0.012$ & $\bm{0.015\pm0.014}$ & $0.020\pm0.014$ \\
         \hline
          \multirow{8}{*}{\datatwo} 
         & \multirow{4}{*}{\shortstack{Gender \\to\\ Race}} 
         & 50 & $0.027\pm0.008$ & $0.041\pm0.006$ & $0.009\pm0.004$ & $\bm{0.001\pm 0.001}$\\ 
         \cline{3-7} & & 100 & $0.027\pm0.008$ & $0.036\pm0.007$ & $0.005\pm0.005$ & $\bm{0.003\pm 0.001}$ \\  
         \cline{3-7} & & 500 & $0.027\pm0.008$ & $0.038\pm0.008$ & $0.003\pm0.002$ & $\bm{0.001\pm 0.001}$ \\ 
         \cline{3-7} & & 1000 & $0.027\pm0.008$ & $0.021\pm0.005$ & $0.006\pm0.005$ & $\bm{0.002\pm 0.001}$ \\
         \cline{2-7} & \multirow{4}{*}{\shortstack{Race \\to\\ Gender}} 
         & 50 & $0.040\pm0.004$ & $0.070\pm0.005$ & $0.035\pm0.004$ & $\bm{0.019\pm0.002}$\\ 
         \cline{3-7} & & 100 & $0.040\pm0.004$ & $0.055\pm0.007$ & $0.034\pm0.003$ & $\bm{0.017\pm0.002}$ \\ 
         \cline{3-7} & & 500 & $0.040\pm0.004$ & $0.042\pm0.008$ & $0.027\pm0.004$ & $\bm{0.019\pm0.002}$ \\ 
         \cline{3-7} & & 1000 & $0.040\pm0.004$ & $0.034\pm0.011$ & $0.028\pm0.004$ & $\bm{0.018\pm0.002}$ \\
         \hline
    \end{tabular}
    \hspace{0.01in}
    \caption{Comparison between the proposed model and the baselines. The numbers in bold indicate the smallest FPR difference achieved in the target domain w.r.t. varying number of target samples.}
    \label{tab:comarison_table}
    \vspace{-0.3in}
\end{table}

\section{Conclusion}

In this paper we provide the first theoretical examination of transfer of machine learning fairness across domains. We adopt a general formulation of domain adaptation for fairness that covers a wide variety of fairness challenges, from proxies of sensitive attributes, to applying models in unanticipated settings.  
Within this general formulation, we have provided theoretical bounds on the transfer of fairness for equal opportunity and equalized odds using both VC-dimension and Rademacher Complexity.  Based on this theory, we developed a new modeling approach to transfer fairness to a given target domain.
In experiments we validate our theoretical results and demonstrate that our modeling approach is more sample efficient in improving fairness metrics in a target domain.

\bibliographystyle{abbrvnat}
\bibliography{references}

\begin{thebibliography}{31}
\providecommand{\natexlab}[1]{#1}
\providecommand{\url}[1]{\texttt{#1}}
\expandafter\ifx\csname urlstyle\endcsname\relax
  \providecommand{\doi}[1]{doi: #1}\else
  \providecommand{\doi}{doi: \begingroup \urlstyle{rm}\Url}\fi

\bibitem[Agarwal et~al.(2018)Agarwal, Beygelzimer, Dud{\'{\i}}k, Langford, and
  Wallach]{DBLP:journals/corr/abs-1803-02453}
A.~Agarwal, A.~Beygelzimer, M.~Dud{\'{\i}}k, J.~Langford, and H.~M. Wallach.
\newblock A reductions approach to fair classification.
\newblock In \emph{Proceedings of the 35th International Conference on Machine
  Learning, {ICML} 2018, Stockholmsm{\"{a}}ssan, Stockholm, Sweden, July 10-15,
  2018}, pages 60--69, 2018.

\bibitem[Ben-David et~al.(2007)Ben-David, Blitzer, Crammer, and
  Pereira]{Ben07:Analysis}
S.~Ben-David, J.~Blitzer, K.~Crammer, and F.~Pereira.
\newblock Analysis of representations for domain adaptation.
\newblock In \emph{Advances in neural information processing systems}, pages
  137--144, 2007.

\bibitem[Ben-David et~al.(2010)Ben-David, Blitzer, Crammer, Kulesza, Pereira,
  and Vaughan]{BenDavid10:Theory}
S.~Ben-David, J.~Blitzer, K.~Crammer, A.~Kulesza, F.~Pereira, and J.~W.
  Vaughan.
\newblock A theory of learning from different domains.
\newblock \emph{Machine learning}, 79\penalty0 (1-2):\penalty0 151--175, 2010.

\bibitem[Beutel et~al.(2017)Beutel, Chen, Zhao, and Chi]{Beutel17:Data}
A.~Beutel, J.~Chen, Z.~Zhao, and E.~H. Chi.
\newblock Data decisions and theoretical implications when adversarially
  learning fair representations.
\newblock \emph{Proceedings of the Conference on Fairness, Accountability and
  Transparency}, 2017.

\bibitem[Beutel et~al.(2019)Beutel, Chen, Doshi, Qian, Woodruff, Luu,
  Kreitmann, Bischof, and Chi]{beutel2019putting}
A.~Beutel, J.~Chen, T.~Doshi, H.~Qian, A.~Woodruff, C.~Luu, P.~Kreitmann,
  J.~Bischof, and E.~H. Chi.
\newblock Putting fairness principles into practice: Challenges, metrics, and
  improvements.
\newblock \emph{Artificial Intelligence, Ethics, and Society}, 2019.

\bibitem[Bousmalis et~al.(2016)Bousmalis, Trigeorgis, Silberman, Krishnan, and
  Erhan]{DBLP:conf/nips/BousmalisTSKE16}
K.~Bousmalis, G.~Trigeorgis, N.~Silberman, D.~Krishnan, and D.~Erhan.
\newblock Domain separation networks.
\newblock In \emph{Advances in Neural Information Processing Systems 29: Annual
  Conference on Neural Information Processing Systems 2016, December 5-10,
  2016, Barcelona, Spain}, pages 343--351, 2016.

\bibitem[Chen et~al.(2018)Chen, Johansson, and Sontag]{chen2018my}
I.~Chen, F.~D. Johansson, and D.~Sontag.
\newblock Why is my classifier discriminatory?
\newblock \emph{arXiv preprint arXiv:1805.12002}, 2018.

\bibitem[Chen et~al.(2019)Chen, Kallus, Mao, Svacha, and
  Udell]{chen2019fairness}
J.~Chen, N.~Kallus, X.~Mao, G.~Svacha, and M.~Udell.
\newblock Fairness under unawareness: Assessing disparity when protected class
  is unobserved.
\newblock In \emph{FAT*}, pages 339--348. ACM, 2019.

\bibitem[Coston et~al.(2019)Coston, Ramamurthy, Wei, Varshney, Speakman,
  Mustahsan, and Chakraborty]{coston2019fair}
A.~Coston, K.~N. Ramamurthy, D.~Wei, K.~R. Varshney, S.~Speakman, Z.~Mustahsan,
  and S.~Chakraborty.
\newblock Fair transfer learning with missing protected attributes.
\newblock In \emph{Proceedings of the AAAI/ACM Conference on Artificial
  Intelligence, Ethics, and Society, Honolulu, HI, USA}, 2019.

\bibitem[Crammer et~al.(2008)Crammer, Kearns, and Wortman]{Crammer08:Learning}
K.~Crammer, M.~Kearns, and J.~Wortman.
\newblock Learning from multiple sources.
\newblock \emph{Journal of Machine Learning Research}, 9\penalty0
  (Aug):\penalty0 1757--1774, 2008.

\bibitem[Dixon et~al.(2018)Dixon, Li, Sorensen, Thain, and
  Vasserman]{dixon2018measuring}
L.~Dixon, J.~Li, J.~Sorensen, N.~Thain, and L.~Vasserman.
\newblock Measuring and mitigating unintended bias in text classification.
\newblock In \emph{available at: www. aies-conference.
  com/wp-content/papers/main/AIES\_2018\_paper\_9. pdf (accessed 6 August
  2018).[Google Scholar]}, 2018.

\bibitem[Edwards and Storkey(2016)]{DBLP:journals/corr/EdwardsS15}
H.~Edwards and A.~J. Storkey.
\newblock Censoring representations with an adversary.
\newblock In \emph{4th International Conference on Learning Representations,
  {ICLR} 2016, San Juan, Puerto Rico, May 2-4, 2016, Conference Track
  Proceedings}, 2016.

\bibitem[Ganin et~al.(2016)Ganin, Ustinova, Ajakan, Germain, Larochelle,
  Laviolette, Marchand, and Lempitsky]{ganin2016domain}
Y.~Ganin, E.~Ustinova, H.~Ajakan, P.~Germain, H.~Larochelle, F.~Laviolette,
  M.~Marchand, and V.~Lempitsky.
\newblock Domain-adversarial training of neural networks.
\newblock \emph{The Journal of Machine Learning Research}, 17\penalty0
  (1):\penalty0 2096--2030, 2016.

\bibitem[Goh et~al.(2016)Goh, Cotter, Gupta, and Friedlander]{NIPS2016_6316}
G.~Goh, A.~Cotter, M.~Gupta, and M.~P. Friedlander.
\newblock Satisfying real-world goals with dataset constraints.
\newblock In D.~D. Lee, M.~Sugiyama, U.~V. Luxburg, I.~Guyon, and R.~Garnett,
  editors, \emph{Advances in Neural Information Processing Systems 29}, pages
  2415--2423. Curran Associates, Inc., 2016.

\bibitem[Gretton et~al.(2012)Gretton, Borgwardt, Rasch, Sch\"{o}lkopf, and
  Smola]{smola12}
A.~Gretton, K.~M. Borgwardt, M.~J. Rasch, B.~Sch\"{o}lkopf, and A.~Smola.
\newblock A kernel two-sample test.
\newblock In \emph{The Journal of Machine Learning Research}, 2012.

\bibitem[Gupta et~al.(2018)Gupta, Cotter, Fard, and
  Wang]{DBLP:journals/corr/abs-1806-11212}
M.~R. Gupta, A.~Cotter, M.~M. Fard, and S.~Wang.
\newblock Proxy fairness.
\newblock \emph{CoRR}, abs/1806.11212, 2018.
\newblock URL \url{http://arxiv.org/abs/1806.11212}.

\bibitem[Hardt et~al.(2016)Hardt, Price, Srebro, et~al.]{hardt2016equality}
M.~Hardt, E.~Price, N.~Srebro, et~al.
\newblock Equality of opportunity in supervised learning.
\newblock In \emph{Advances in neural information processing systems}, pages
  3315--3323, 2016.

\bibitem[Kallus and Zhou(2018)]{DBLP:conf/icml/KallusZ18}
N.~Kallus and A.~Zhou.
\newblock Residual unfairness in fair machine learning from prejudiced data.
\newblock In \emph{Proceedings of the 35th International Conference on Machine
  Learning, {ICML} 2018, Stockholmsm{\"{a}}ssan, Stockholm, Sweden, July 10-15,
  2018}, pages 2444--2453, 2018.

\bibitem[Kallus and Zhou(2019)]{kallus2019assessing}
N.~Kallus and A.~Zhou.
\newblock Assessing disparate impacts of personalized interventions:
  Identifiability and bounds.
\newblock \emph{arXiv preprint arXiv:1906.01552}, 2019.

\bibitem[Lan and Huan(2017)]{Lan17:Discrimantory}
C.~Lan and J.~Huan.
\newblock Discriminatory transfer.
\newblock \emph{CoRR}, 2017.
\newblock URL \url{http://arxiv.org/abs/1707.00780}.

\bibitem[Li et~al.(2018)Li, Baldwin, and Cohn]{Li18:Towards}
Y.~Li, T.~Baldwin, and T.~Cohn.
\newblock Towards robust and privacy-preserving text representations.
\newblock \emph{arXiv preprint arXiv:1805.06093}, 2018.

\bibitem[Long et~al.(2015)Long, Cao, Wang, and Jordan]{long15}
M.~Long, Y.~Cao, J.~Wang, and M.~Jordan.
\newblock Learning transferable features with deep adaptation networks.
\newblock In \emph{Proceedings of the 32nd International Conference on
  International Conference on Machine Learning}, 2015.

\bibitem[Louizos et~al.(2016)Louizos, Swersky, Li, Welling, and
  Zemel]{DBLP:journals/corr/LouizosSLWZ15}
C.~Louizos, K.~Swersky, Y.~Li, M.~Welling, and R.~S. Zemel.
\newblock The variational fair autoencoder.
\newblock In \emph{4th International Conference on Learning Representations,
  {ICLR} 2016, San Juan, Puerto Rico, May 2-4, 2016, Conference Track
  Proceedings}, 2016.

\bibitem[Madras et~al.(2018)Madras, Creager, Pitassi, and
  Zemel]{Madras18:Learning}
D.~Madras, E.~Creager, T.~Pitassi, and R.~Zemel.
\newblock Learning adversarially fair and transferable representations.
\newblock \emph{arXiv preprint arXiv:1802.06309}, 2018.

\bibitem[Mansour et~al.(2009)Mansour, Mohri, and
  Rostamizadeh]{Mansour09:Domain}
Y.~Mansour, M.~Mohri, and A.~Rostamizadeh.
\newblock Domain adaptation: Learning bounds and algorithms.
\newblock \emph{COLT}, 2009.

\bibitem[Pan et~al.(2010)Pan, Yang, et~al.]{Pan10:Survey}
S.~J. Pan, Q.~Yang, et~al.
\newblock A survey on transfer learning.
\newblock \emph{IEEE Transactions on knowledge and data engineering},
  22\penalty0 (10):\penalty0 1345--1359, 2010.

\bibitem[Sculley et~al.(2015)Sculley, Holt, Golovin, Davydov, Phillips, Ebner,
  Chaudhary, Young, Crespo, and Dennison]{sculley2015hidden}
D.~Sculley, G.~Holt, D.~Golovin, E.~Davydov, T.~Phillips, D.~Ebner,
  V.~Chaudhary, M.~Young, J.-F. Crespo, and D.~Dennison.
\newblock Hidden technical debt in machine learning systems.
\newblock In \emph{Advances in neural information processing systems}, pages
  2503--2511, 2015.

\bibitem[Tsipras et~al.(2018)Tsipras, Santurkar, Engstrom, Turner, and
  Madry]{Tsipras18:There}
D.~Tsipras, S.~Santurkar, L.~Engstrom, A.~Turner, and A.~Madry.
\newblock There is no free lunch in adversarial robustness (but there are
  unexpected benefits).
\newblock \emph{arXiv preprint arXiv:1805.12152}, 2018.

\bibitem[Weiss et~al.(2016)Weiss, Khoshgoftaar, and Wang]{Weiss16:Survey}
K.~Weiss, T.~M. Khoshgoftaar, and D.~Wang.
\newblock A survey of transfer learning.
\newblock \emph{Journal of Big Data}, 2016.

\bibitem[Zafar et~al.(2017)Zafar, Valera, Gomez{-}Rodriguez, and
  Gummadi]{DBLP:conf/aistats/ZafarVGG17}
M.~B. Zafar, I.~Valera, M.~Gomez{-}Rodriguez, and K.~P. Gummadi.
\newblock Fairness constraints: Mechanisms for fair classification.
\newblock In \emph{Proceedings of the 20th International Conference on
  Artificial Intelligence and Statistics, {AISTATS} 2017, 20-22 April 2017,
  Fort Lauderdale, FL, {USA}}, pages 962--970, 2017.

\bibitem[Zhang et~al.(2018)Zhang, Lemoine, and Mitchell]{Zhang18:Mitigating}
B.~H. Zhang, B.~Lemoine, and M.~Mitchell.
\newblock Mitigating unwanted biases with adversarial learning.
\newblock \emph{CoRR}, abs/1801.07593, 2018.
\newblock URL \url{http://arxiv.org/abs/1801.07593}.

\end{thebibliography}

\clearpage

\appendix

\section{Rademacher Complexity}\label{sec:radmacher}

We provide additional bounds dependent on Radmacher Complexity based on the following definition of data-driven empirical Rademacher Complexity

\begin{definition}
Given a hypothesis space $\mathcal{H}$, a sample $S\in\mathcal{X}^m$, the empirical Rademacher Complexity of $\mathcal{H}$ is defined as
\begin{equation*}
\hat{\mathfrak{R}}_S(\mathcal{H})=\frac{2}{m}\mathbb{E}_\sigma\left[ \sup_{h\in\mathcal{H}} \left. | \sum_{i=1}^m \sigma_ih(x_i) | \right| S = (x_1, \ldots, x_m)  \right].
\end{equation*}
The expectation is taken over $\sigma=(\sigma_1, \ldots, \sigma_m)$ where $\sigma_i\in\{-1,+1\}$ are uniform independent random variables. The Rademacher Complexity of a hypothesis space is defined as the expectation of $\hat{\mathfrak{R}}$ over all sample sets of size $m$
\begin{equation}
\mathfrak{R}_m(\mathcal{H}) = \mathbb{E}_S\left[ \left. \hat{\mathfrak{R}}_S(\mathcal{H}) \right| |S|=m\right].
\end{equation}
\end{definition}

Rademacher Complexity measures the ability of a hypothesis space to fit random noise. The empirical Rademacher Complexity function allows us to estimate the Rademacher Complexity using a finite sample of data. Rademacher Complexity bounds can lead to tighter bounds than those of VC-dimension, especially when analyzing neural network models.

When transitioning to Rademacher Complexity we need to change the binary labels from $\{0,1\}$ to $\{-1,1\}$. This means that the error of a hypothesis $g$ is defined as
$$
\e_{S_\alpha^l}(g,f) = \mathbb{E}_{z_\alpha^l \sim D_{S_\alpha^l}} \left[ \frac{|g(z_\alpha^l) - f(z_\alpha^l)|}{2} \right].
$$

Additionally, we need new definitions of the equal opportunity and equalized odds distances over the new binary group membership. The equal opportunity distance is defined as
\begin{align*}
\eopdiff{S}(g) \triangleq&\ \mathbb{E}_{Z_0^{-1} \sim D_{S_0^{-1}}} \left[\frac{1+g(z_0^{-1})}{2}\right] - \mathbb{E}_{Z_1^{-1} \sim D_{S_1^{-1}}} \left[\frac{1+g(z_1^{-1})}{2}\right],
\end{align*}
while the equlized odds distance is defined as
\begin{align*}
\eodiff{T}(g) \triangleq&\ \left| \mathbb{E}_{Z_0^{-1}\sim D_{T_0^{-1}}}\left[\frac{1+g(z_0^{-1})}{2}\right] - \mathbb{E}_{Z_1^{-1}\sim D_{T_1^{-1}}}\left[\frac{1+g(z_1^{-1})}{2}\right] \right| \\
&\ + \left| \mathbb{E}_{Z_0^{1}\sim D_{T_0^{1}}}\left[\frac{1+g(z_0^{1})}{2}\right] - \mathbb{E}_{Z_1^{1}\sim D_{T_1^{1}}}\left[\frac{1+g(z_1^{1})}{2}\right] \right|.
\end{align*}

Using these new definitions Theorem~\ref{thm:eop_transfer_bound_rademacher_app} provides a Rademacher Complexity bound of the equal opportunity distance in the target space. This closely resembles the VC-dimension bound in Theorem~\ref{thm:eop_transfer_bound}.

\begin{theorem}\label{thm:eop_transfer_bound_rademacher_app}
Let $\mathcal{H}$ be a hypothesis space. If $\ \mathcal{U}_{S_0^{-1}},\ \mathcal{U}_{S_1^{-1}},\ \mathcal{U}_{T_0^{-1}},\  \mathcal{U}_{T_1^{-1}}$ are samples of size $m'$, each drawn from $\mathcal{D}_{S_0^{-1}}$, $\mathcal{D}_{S_1^{-1}}$, $\mathcal{D}_{T_0^{-1}}$, and $\mathcal{D}_{T_1^{-1}}$ respectively, then for any $\delta\in(0,1)$, with probability at least $1-\delta$ (over the choice of samples), for every $g\in\mathcal{H}$ (where $\mathcal{H}$ is a symmetric hypothesis space) the distance from equal opportunity in the target space is bounded by
\begin{align*}
    \eopdiff{T}(g) \leq&\ \eopdiff{S}(g) + \frac{1}{2}\hat{d}_{\symhyp}(\mathcal{U}_{T_0^{-1}},\mathcal{U}_{S_0^{-1}}) + \frac{1}{2}\hat{d}_{\symhyp}(\mathcal{U}_{T_1^{-1}},\mathcal{U}_{S_1^{-1}}) \\
    &\ + 2\left(\mathfrak{R}_{U_{T_0^{-1}}}(\mathcal{H}) + \mathfrak{R}_{U_{S_0^{-1}}}(\mathcal{H}) + \mathfrak{R}_{U_{T_1^{-1}}}(\mathcal{H}) + \mathfrak{R}_{U_{S_1^{-1}}}(\mathcal{H})\right) \\
    &\ + 6\sqrt{\frac{\log\frac{2}{\delta}}{2m}}
    + \lambda_0^{-1} + \lambda_1^{-1},
\end{align*}
where $\lambda_\alpha^l=\e_{S_\alpha^l}(g^*,f)+\e_{T_\alpha^l}(g^*,f)$.
\end{theorem}

The proof also follows a similar logic to the sketch given for Theorem~\ref{thm:eop_transfer_bound} with the additional step of using a modification of Corollary 7 given by \citet{Mansour09:Domain}. 

Similarly, Theorem~\ref{thm:eo_transfer_bound_rademacher_app} provides a Rademacher Complexity bound of the equalized odds distance in the target space.

\begin{theorem}\label{thm:eo_transfer_bound_rademacher_app}
Let $\mathcal{H}$ be a hypothesis space. If $\ \mathcal{U}_{S_0^{-1}},\ \mathcal{U}_{S_1^{-1}},\ \mathcal{U}_{T_0^{-1}},\  \mathcal{U}_{T_1^{-1}} \ \mathcal{U}_{S_0^{1}},\ \mathcal{U}_{S_1^{1}},\ \mathcal{U}_{T_0^{1}},\  \mathcal{U}_{T_1^{1}}$ are samples of size $m'$, each drawn from $\mathcal{D}_{S_0^{-1}}$, $\mathcal{D}_{S_1^{-1}}$, $\mathcal{D}_{T_0^{-1}}$, $\mathcal{D}_{T_1^{-1}}, \mathcal{D}_{S_0^{1}}$, $\mathcal{D}_{S_1^{1}}$, $\mathcal{D}_{T_0^{1}}$, and $\mathcal{D}_{T_1^{1}}$ respectively, then for any $\delta\in(0,1)$, with probability at least $1-\delta$ (over the choice of samples), for every $g\in\mathcal{H}$ (where $\mathcal{H}$ is a symmetric hypothesis space) the distance from equalized odds in the target space is bounded by
\begin{align*}
    \eodiff{T}(g) \leq&\ \eodiff{S}(g) + \frac{1}{2}\left( \hat{d}_{\symhyp}(\mathcal{U}_{S_0^{-1}}, \mathcal{U}_{T_0^{-1}}) + \hat{d}_{\symhyp}(\mathcal{U}_{S_1^{-1}}, \mathcal{U}_{T_1^{-1}}) \right. \\
    &\ \left. + \hat{d}_{\symhyp}(\mathcal{U}_{S_0^{1}}, \mathcal{U}_{T_0^{1}}) + \hat{d}_{\symhyp}(\mathcal{U}_{S_1^{1}}, \mathcal{U}_{T_1^{1}}) \right) \\
    &\ + 2\left( \hat{\mathfrak{R}}_{U_{S_0^{-1}}}(\mathcal{H}) + \hat{\mathfrak{R}}_{U_{T_0^{-1}}}(\mathcal{H}) \right. + \hat{\mathfrak{R}}_{U_{S_1^{-1}}}(\mathcal{H}) + \hat{\mathfrak{R}}_{U_{T_1^{-1}}}(\mathcal{H}) \\
    &\ + \hat{\mathfrak{R}}_{U_{S_0^{1}}}(\mathcal{H}) + \hat{\mathfrak{R}}_{U_{T_0^{1}}}(\mathcal{H}) \left. + \hat{\mathfrak{R}}_{U_{S_1^{1}}}(\mathcal{H}) + \hat{\mathfrak{R}}_{U_{T_1^{1}}}(\mathcal{H}) \right) \\
    &\ + 12\sqrt{\frac{\log\frac{2}{\delta}}{2m}} + \lambda_\mathit{EO},
\end{align*}
where $\lambda_\mathit{EO} = \lambda_0^{-1} + \lambda_1^{-1} + \lambda_0^1 + \lambda_1^1$, and $\lambda_\alpha^l=\e_{S_\alpha^l}(g^*,f)+ \e_{T_\alpha^l}(g^*,f)$.
\end{theorem}

Given either the Rademacher Complexity bounds or the VC-dimension bounds, the implications stay the same. In order for a successful transfer of fairness the two (or four) subspace domains should be close across the source and target domains. Additionally, there should be a hypothesis in the hypothesis space that performs well over all of the relevant subspaces.

\section{Proofs}\label{app:proofs}

\setcounter{theorem}{0}

\begin{lemma} \label{lem:sym_hyp}
(From \citet{BenDavid10:Theory})
For any hypotheses $h,h'\in\mathcal{H}$,
$$
\left| \e_S(h,h') - \e_T(h,h') \right| \leq \frac{1}{2}d_{\symhyp}(D_S,D_T).
$$
\end{lemma}

\begin{lemma}\label{lem:triangle}
(From \cite{Ben07:Analysis,Crammer08:Learning}) 
For any labeling functions $f_1$, $f_2$, and $f_3$, we have
$$
\e(f_1,f_2) \leq \e(f_1,f_3) + \e(f_2,f_3).
$$
\end{lemma}

\subsection{VC-dimension bounds}

\begin{lemma}\label{lem:empirical}
(From \citet{BenDavid10:Theory})
Let $\mathcal{H}$ be a hypothesis space on $\mathcal{Z}$ with VC-dimension $d$. If $\mathcal{U}$ and $\mathcal{U'}$ are samples of size $m$ from $\mathcal{D}$ and $\mathcal{D}'$ respectively and $\hat{d}_\mathcal{H}(\mathcal{U},\mathcal{U}')$ is the empirical $\mathcal{H}$-divergence between samples, then for any $\delta\in(0,1)$, with probability at least $1-\delta$,
$$
d_\mathcal{H}(\mathcal{D},\mathcal{D}') \leq \hat{d}_\mathcal{H}(\mathcal{U},\mathcal{U}') + 4\sqrt{\frac{d\log(2m)+\log(\frac{2}{\delta})}{m}}.
$$
\end{lemma}

\begin{theorem}
Let $\mathcal{H}$ be a hypothesis space of VC dimension $d$. If $\ \mathcal{U}_{S_0^0},\ \mathcal{U}_{S_1^0},\ \mathcal{U}_{T_0^1},\  \mathcal{U}_{T_1^0}$ are samples of size $m'$ each, drawn from $\mathcal{D}_{S_0^0}$, $\mathcal{D}_{S_1^0}$, $\mathcal{D}_{T_0^0}$, and $\mathcal{D}_{T_1^0}$ respectively, then for any $\delta\in(0,1)$, with probability at least $1-\delta$ (over the choice of samples), for every $g\in\mathcal{H}$ (where $\mathcal{H}$ is a symmetric hypothesis space) the distance from equal opportunity in the target space is bounded by
\begin{align*}
    \eopdiff{T}(g) \leq&\ \eopdiff{S}(g) + \frac{1}{2}\hat{d}_{\symhyp}(\mathcal{U}_{T_0^0},\mathcal{U}_{S_0^0}) + \frac{1}{2}\hat{d}_{\symhyp}(\mathcal{U}_{T_1^0},\mathcal{U}_{S_1^0}) \\
    &\ + 8 \sqrt{\frac{2d\log(2m')+\log(\frac{2}{\delta})}{m'}} + \lambda_0^{0} + \lambda_1^0,
\end{align*}
where $\lambda_\alpha^l=\e_{S_\alpha^l}(g^*,f)+\e_{T_\alpha^l}(g^*,f)$.
\end{theorem}

\begin{proof}
Without loss of generality assume $\mathbb{E}_{Z_0^0 \sim D_{S_0^0}} \geq \mathbb{E}_{Z_1^0 \sim D_{S_1^0}}$. Then we can rewrite $\eopdiff{S}(g)$ as follows:
\begin{align*}
    \eopdiff{S}(g) & = \mathbb{E}_{Z_0^0 \sim \mathcal{D}_{S_0^0}} \left[ g(Z_0^0)\right]-\mathbb{E}_{Z_1^0 \sim \mathcal{D}_{S_1^0}} \left[ g(z_1^0) \right] \\
    &= \mathbb{E}_{Z_0^0 \sim \mathcal{D}_{S_0^0}} \left[ g(Z_0^0)\right]+\mathbb{E}_{Z_1^0 \sim \mathcal{D}_{S_1^0}} \left[ 1- g(z_1^0) \right] - 1 \\
    &= \e_{S_0^0}(g,f) + \e_{S_1^0}(1-g,f)-1,
\end{align*}
where the last line follows from the fact that equal opportunity only cares about the error on the false data-points.

We now have the tools to find an upper-bound on $\eopdiff{T}(g)$.
\begin{align}
    \eopdiff{T}(g) =&{} \e_{T_0^0}(g,f) + \e_{T_1^0}(1-g,f)-1 \nonumber \\
    \leq&\ \e_{T_0^0}(g,g^*) + \e_{T_0^0}(f,g^*)  +  \e_{T_1^0}(1-g,g^*) + \e_{T_1^0}(f,g^*)-1 \label{eq:eop_1} \\
    =&\ \e_{T_0^0}(g^*,f) + \e_{T_0^0}(g,g^*) + \e_{T_1^0}(g^*,f) + \e_{T_1^0}(1-g,g^*)-1 \nonumber \\
    =&\ \e_{T_0^0}(g^*,f) + \e_{T_0^0}(g,g^*) + \e_{S_0^0}(g,g^*) - \e_{S_0^0}(g,g^*) \nonumber \\
    &\ + \e_{T_1^0}(g^*,f) + \e_{T_1^0}(1-g,g^*) + \e_{S_1^0}(1-g,g^*) -\e_{S_1^0}(1-g,g^*) - 1 \nonumber \\
    \leq&\ \e_{T_0^0}(g^*,f) + \e_{S_0^0}(g,g^*)  + \left| \e_{T_0^0}(g,g^*)  - \e_{S_0^0}(g,g^*) \right| \nonumber \\
    &\ + \e_{T_1^0}(g^*,f) + \e_{S_1^0}(1-g,g^*) + \left| \e_{T_1^0}(1-g,g^*) - \e_{S_1^0}(1-g,g^*) \right| - 1 \nonumber \\
    \leq&\ \e_{T_0^0}(g^*,f) + \e_{S_0^0}(g,g^*) + \frac{1}{2}d_{\symhyp}(D_{T_0^0},D_{S_0^0}) \nonumber \\
    &\ + \e_{T_1^0}(g^*,f) + \e_{S_1^0}(1-g,g^*) +  \frac{1}{2}d_{\symhyp}(D_{T_1^0},D_{S_1^0}) - 1  \label{eq:eop_2}\\
    \leq&\ \e_{T_0^0}(g^*,f) + \e_{S_0^0}(g,f) + \e_{S_0^0}(g^*,f) + \frac{1}{2}d_{\symhyp}(D_{T_0^0},D_{S_0^0})  \nonumber \\
    &\ + \e_{T_1^0}(g^*,f) + \e_{S_1^0}(1-g,f) + \e_{S_1^0}(g^*,f) + \frac{1}{2}d_{\symhyp}(D_{T_1^0},D_{S_1^0}) - 1 \label{eq:eop_3} \\
    =&\ \e_{S_0^0}(g,f) + \e_{T_0^0}(g^*,f) + \e_{S_0^0}(g^*,f)  + \frac{1}{2}d_{\symhyp}(D_{T_0^0},D_{S_0^0}) \nonumber \\
    &\ + \e_{S_1^0}(1-g,f) + \e_{T_1^0}(g^*,f) +  \e_{S_1^0}(g^*,f) + \frac{1}{2}d_{\symhyp}(D_{T_1^0},D_{S_1^0}) - 1 \nonumber \\
    =&\ \e_{S_0^0}(g,f) + \e_{S_1^0}(1-g,f) - 1 + \frac{1}{2}d_{\symhyp}(D_{T_0^0},D_{S_0^0}) \nonumber \\
    &\ + \frac{1}{2}d_{\symhyp}(D_{T_1^0},D_{S_1^0}) + \lambda_0^0 + \lambda_1^0 \label{eq:eop_4} \\
    =&\ \eopdiff{S}(g) + \frac{1}{2}d_{\symhyp}(D_{T_0^0},D_{S_0^0}) + \frac{1}{2}d_{\symhyp}(D_{T_1^0},D_{S_1^0}) + \lambda_0^0 + \lambda_1^0 \nonumber \\
    \leq&\ \eopdiff{S}(g) + \frac{1}{2}\hat{d}_{\symhyp}(\mathcal{U}_{T_0^0},\mathcal{U}_{S_0^0}) + \frac{1}{2}\hat{d}_{\symhyp}(\mathcal{U}_{T_1^0},\mathcal{U}_{S_1^0}) \nonumber \\
    &\ + 8 \sqrt{\frac{2d\log(2m')+\log(\frac{2}{\delta})}{m'}} + \lambda_0^0 + \lambda_1^0 \label{eq:eop_6},
\end{align}

Where inequality \ref{eq:eop_1} is due to lemma \ref{lem:triangle}, inequality \ref{eq:eop_2} is due to lemma \ref{lem:sym_hyp} and the fact that $\mathcal{H}$ is a symmetric hypothesis space, inequality \ref{eq:eop_3} is due to lemma \ref{lem:triangle}, equality \ref{eq:eop_4} is due to the definition of $\lambda_\alpha^l$, and inequality \ref{eq:eop_6} is due to lemma \ref{lem:empirical}.
\end{proof}

\begin{theorem}
Let $\mathcal{H}$ be a hypothesis space of VC dimension $d$. If $\ \mathcal{U}_{S_\alpha^l}$ are samples of size $m'$ each, drawn from $\mathcal{D}_{S_\alpha^l}$ for all  $\alpha\in \Omega_A=\{0,1\}$ and  $l\in \Omega_\mathcal{Y}={0,1}$, then for any $\delta\in(0,1)$, with probability at least $1-\delta$ (over the choice of samples), for every $g\in\mathcal{H}$ (where $\mathcal{H}$ is a symmetric hypothesis space) the distance from equalized odds in the target space is bounded by
\begin{align*}
    \eodiff{T}(g) \leq&\ \eodiff{S}(g) + \frac{1}{2}\hat{d}_{\symhyp}(\mathcal{U}_{T_0^0},\mathcal{U}_{S_0^0}) + \frac{1}{2}\hat{d}_{\symhyp}(\mathcal{U}_{T_1^0},\mathcal{U}_{S_1^0}) \\
    &\ + \frac{1}{2}\hat{d}_{\symhyp}(\mathcal{U}_{T_0^1},\mathcal{U}_{S_0^1}) + \frac{1}{2}\hat{d}_{\symhyp}(\mathcal{U}_{T_1^1},\mathcal{U}_{S_1^1})  \\
    &\ + 16 \sqrt{\frac{2d\log(2m')+\log(\frac{2}{\delta})}{m'}} + \lambda_\mathit{EO},
\end{align*}
where $\lambda_\mathit{EO} = \lambda_0^0 + \lambda_1^0 + \lambda_0^1 + \lambda_1^1$, and $\lambda_\alpha^l=\e_{S_\alpha^l}(g^*,f)+\e_{T_\alpha^l}(g^*,f)$.
\end{theorem}

\begin{proof}
WLOG assume $\mathbb{E}_{Z_0^0 \sim D_{S_0^0}}[g] \geq \mathbb{E}_{Z_1^0 \sim D_{S_1^0}}[g]$ and $\mathbb{E}_{Z_0^1 \sim D_{S_0^1}}[g] \geq \mathbb{E}_{Z_1^1 \sim D_{S_1^1}}[g]$. Then,
\begin{align*}
    \eodiff{S} =&\ \mathbb{E}_{Z_0^0 \sim D_{S_0^0}}[g] - \mathbb{E}_{Z_1^0 \sim D_{S_1^0}}[g]  + \mathbb{E}_{Z_0^1 \sim D_{S_0^1}}[g] - \mathbb{E}_{Z_1^1 \sim D_{S_1^1}}[g] \\
    =&\ \mathbb{E}_{Z_0^0 \sim D_{S_0^0}}[g] + \mathbb{E}_{Z_1^0 \sim D_{S_1^0}}[1-g]  + \mathbb{E}_{Z_0^1 \sim D_{S_0^1}}[g] + \mathbb{E}_{Z_1^1 \sim D_{S_1^1}}[1-g] - 2 \\
    =&\ \e_{S_0^0}(g,f) + \e_{S_1^0}(1-g,f) + \e_{S_0^1}(g,f) + \e_{S_1^1}(1-g,f) - 2.
\end{align*}

Using this and the previous lemmas we have:
\begin{align}
    \eodiff{T}(g) =&\ \e_{T_0^0}(g,f) + \e_{T_1^0}(1-g,f) + \e_{T_0^1}(g,f) + \e_{T_1^1}(1-g,f) - 2 \nonumber\\
    \leq&\ \e_{T_0^0}(g,g^*) + \e_{T_0^0}(f,g^*) + \e_{T_1^0}(1-g,g^*) + \e_{T_1^0}(f,g^*) \nonumber \\
        &\ + \e_{T_0^1}(g,g^*) + \e_{T_0^1}(f,g^*) + \e_{T_1^1}(1-g,g^*) + \e_{T_1^1}(f,g^*) - 2 \label{eq:eo_1}\\
    =&\ \e_{T_0^0}(g^*,f) + \e_{T_0^0}(g,g^*) + \e_{S_0^0}(g,g^*) - \e_{S_0^0}(g,g^*) \nonumber \\
        &\ + \e_{T_1^0}(g^*,f) + \e_{T_1^0}(1-g,g^*) + \e_{S_1^0}(1-g,g^*) - \e_{S_1^0}(1-g,g^*) \nonumber \\
        &\ + \e_{T_0^1}(g^*,f) + \e_{T_0^1}(g,g^*) + \e_{S_0^1}(g,g^*) - \e_{S_0^1}(g,g^*) \nonumber \\
        &\ + \e_{T_1^1}(f,g^*) + \e_{T_1^1}(1-g,g^*) + \e_{S_1^1}(1-g,g^*) - \e_{S_1^1}(1-g,g^*) - 2 \nonumber \\
    \leq&\ \e_{T_0^0}(g^*,f) + \e_{S_0^0}(g,g^*) + \left | \e_{T_0^0}(g,g^*) - \e_{S_0^0}(g,g^*) \right| \nonumber \\
        &\ + \e_{T_1^0}(g^*,f) + \e_{S_1^0}(1-g,g^*) + \left| \e_{T_1^0}(1-g,g^*) - \e_{S_1^0}(1-g,g^*) \right| \nonumber \\
        &\ + \e_{T_0^1}(g^*,f) + \e_{S_0^1}(g,g^*) + \left| \e_{T_0^1}(g,g^*) - \e_{S_0^1}(g,g^*) \right| \nonumber \\
        &\ + \e_{T_1^1}(f,g^*) + \e_{S_1^1}(1-g,g^*) + \left| \e_{T_1^1}(1-g,g^*) - \e_{S_1^1}(1-g,g^*) \right| - 2 \nonumber \\
    \leq&\ \e_{T_0^0}(g^*,f) + \e_{S_0^0}(g,g^*) + \frac{1}{2}d_{\symhyp}(D_{T_0^0},D_{S_0^0}) \nonumber \\
        &\ + \e_{T_1^0}(g^*,f) + \e_{S_1^0}(1-g,g^*) + \frac{1}{2}d_{\symhyp}(D_{T_1^0},D_{S_1^0}) \nonumber \\
        &\ + \e_{T_0^1}(g^*,f) + \e_{S_0^1}(g,g^*)  + \frac{1}{2}d_{\symhyp}(D_{T_0^1},D_{S_0^1}) \nonumber \\
        &\ + \e_{T_1^1}(f,g^*) + \e_{S_1^1}(1-g,g^*) + \frac{1}{2}d_{\symhyp}(D_{T_1^1},D_{S_1^1}) - 2 \label{eq:eo_2} \\
    \leq&\ \e_{T_0^0}(g^*,f) + \e_{S_0^0}(g,f) + \e_{S_0^0}(g^*,f) + \frac{1}{2}d_{\symhyp}(D_{T_0^0},D_{S_0^0}) \nonumber \\
        &\ + \e_{T_1^0}(g^*,f) + \e_{S_1^0}(1-g,f) + \e_{S_1^0}(g^*,f) + \frac{1}{2}d_{\symhyp}(D_{T_1^0},D_{S_1^0}) \nonumber \\
        &\ + \e_{T_0^1}(g^*,f) + \e_{S_0^1}(g,f) + \e_{S_0^1}(g^*,f) + \frac{1}{2}d_{\symhyp}(D_{T_0^1},D_{S_0^1}) \nonumber \\
        &\ + \e_{T_1^1}(f,g^*) + \e_{S_1^1}(1-g,f) + \e_{S_1^1}(g^*,f) + \frac{1}{2}d_{\symhyp}(D_{T_1^1},D_{S_1^1}) - 2 \label{eq:eo_3} \\
    =&\ \lambda_0^0 + \e_{S_0^0}(g,f) + \frac{1}{2}d_{\symhyp}(D_{T_0^0},D_{S_0^0}) \nonumber \\
        &\ + \lambda_1^0 + \e_{S_1^0}(1-g,f) + \frac{1}{2}d_{\symhyp}(D_{T_1^0},D_{S_1^0}) \nonumber \\
        &\ + \lambda_0^1 + \e_{S_0^1}(g,f) + \frac{1}{2}d_{\symhyp}(D_{T_0^1},D_{S_0^1}) \nonumber \\
        &\ + \lambda_1^1 + \e_{S_1^1}(1-g,f) + \frac{1}{2}d_{\symhyp}(D_{T_1^1},D_{S_1^1}) - 2 \nonumber \\
    =&\ \eodiff{S}(g) + \frac{1}{2}d_{\symhyp}(D_{T_0^0},D_{S_0^0}) + \frac{1}{2}d_{\symhyp}(D_{T_1^0},D_{S_1^0}) \nonumber \\
        &\ + \frac{1}{2}d_{\symhyp}(D_{T_0^1},D_{S_0^1}) + \frac{1}{2}d_{\symhyp}(D_{T_1^1},D_{S_1^1}) + \lambda_\mathit{EO} \nonumber \\
    \leq&\ \eodiff{S}(g) + \frac{1}{2}\hat{d}_{\symhyp}(\mathcal{U}_{T_0^0},\mathcal{U}_{S_0^0}) + \frac{1}{2}\hat{d}_{\symhyp}(\mathcal{U}_{T_1^0},\mathcal{U}_{S_1^0}) \nonumber \\
        &\ + \frac{1}{2}\hat{d}_{\symhyp}(\mathcal{U}_{T_0^1},\mathcal{U}_{S_0^1}) + \frac{1}{2}\hat{d}_{\symhyp}(\mathcal{U}_{T_1^1},\mathcal{U}_{S_1^1}) \nonumber \\
        &\ + 16 \sqrt{\frac{2d\log(2m')+\log(\frac{2}{\delta})}{m'}} + \lambda_\mathit{EO} \label{eq:eo_5},
\end{align}
where inequality \ref{eq:eo_1} is due to lemma \ref{lem:triangle}, inequality \ref{eq:eo_2} is due to lemma \ref{lem:sym_hyp} and the fact that $\mathcal{H}$ is a symmetric hypothesis space, inequality \ref{eq:eo_3} is due to lemma \ref{lem:triangle}, and inequality \ref{eq:eo_5} is due to lemma \ref{lem:empirical}.
\end{proof}

\subsection{Rademacher Complexity Bounds}

\begin{lemma} \label{lem:rademacher_dist}
(A modification of Corollary 7 from \citet{Mansour09:Domain})
Let $\mathcal{H}$ by a hypothesis set of classifiers mapping the feature space $X$ to the labels $\{-1,1\}$. Let $\mathcal{U}$ and $\mathcal{U}'$ be the set of samples each of size $m$ sampled from $\distro$ and $\distro'$ respectively. Then, for any $\delta>0$, with probability at least $1-\delta$ over samples $\mathcal{U}$ and $\mathcal{U}'$:
\begin{align*}
    d_{\mathcal{H}}(\distro,\distro') & \leq \hat{d}_{\mathcal{H}}(\mathcal{U}, \mathcal{U}') + 4\left(\mathfrak{R}_{\mathcal{U}}(\mathcal{H}) + \mathfrak{R}_{\mathcal{U}}(\mathcal{H})\right) + 3\sqrt{\frac{\log\frac{2}{\delta}}{2m}}
\end{align*}
\end{lemma}

\begin{theorem}
Let $\mathcal{H}$ be a hypothesis space. If $\ \mathcal{U}_{S_0^{-1}},\ \mathcal{U}_{S_1^{-1}},\ \mathcal{U}_{T_0^{-1}},\  \mathcal{U}_{T_1^{-1}}$ are samples of size $m'$ each, drawn from $\mathcal{D}_{S_0^{-1}}$, $\mathcal{D}_{S_1^{-1}}$, $\mathcal{D}_{T_0^{-1}}$, and $\mathcal{D}_{T_1^{-1}}$ respectively, then for any $\delta\in(0,1)$, with probability at least $1-\delta$ (over the choice of samples), for every $g\in\mathcal{H}$ (where $\mathcal{H}$ is a symmetric hypothesis space) the distance from equal opportunity in the target space is bounded by
\begin{align*}
    \eopdiff{T}(g) \leq&\ \eopdiff{S}(g) + \frac{1}{2}\hat{d}_{\symhyp}(\mathcal{U}_{T_0^{-1}},\mathcal{U}_{S_0^{-1}}) + \frac{1}{2}\hat{d}_{\symhyp}(\mathcal{U}_{T_1^{-1}},\mathcal{U}_{S_1^{-1}}) \\
    &\ + 2\left(\mathfrak{R}_{U_{T_0^{-1}}}(\mathcal{H}) + \mathfrak{R}_{U_{S_0^{-1}}}(\mathcal{H}) + \mathfrak{R}_{U_{T_1^{-1}}}(\mathcal{H}) + \mathfrak{R}_{U_{S_1^{-1}}}(\mathcal{H})\right) \\
    &\ + 6\sqrt{\frac{\log\frac{2}{\delta}}{2m}}
    + \lambda_0^{-1} + \lambda_1^{-1},
\end{align*}
where $\lambda_\alpha^l=\e_{S_\alpha^l}(g^*,f)+\e_{T_\alpha^l}(g^*,f)$.
\end{theorem}

\begin{proof}
Without loss of generality assume $\mathbb{E}_{Z_0^{-1} \sim D_{S_0^{-1}}} \geq \mathbb{E}_{Z_1^{-1} \sim D_{S_1^{-1}}}$. Then we can rewrite $\eopdiff{S}$ as follows.
\begin{align}
    \eopdiff{S}(g) =&\ \mathbb{E}_{Z_0^{-1} \sim D_{S_0^{-1}}} \left[\frac{1+g(z_0^{-1})}{2}\right] - \mathbb{E}_{Z_1^{-1} \sim D_{S_1^{-1}}} \left[\frac{1+g(z_1^{-1})}{2}\right] \nonumber \\
    =&\ \mathbb{E}_{Z_0^{-1} \sim D_{S_0^{-1}}} \left[\frac{1+g(z_0^{-1})}{2}\right] + \mathbb{E}_{Z_1^{-1} \sim D_{S_1^{-1}}} \left[1-\frac{1+g(z_1^{-1})}{2}\right] - 1 \nonumber \\
    =&\ \mathbb{E}_{Z_0^{-1} \sim D_{S_0^{-1}}} \left[\frac{1+g(z_0^{-1})}{2}\right] + \mathbb{E}_{Z_1^{-1} \sim D_{S_1^{-1}}} \left[\frac{1-g(z_1^{-1})}{2}\right] - 1 \nonumber \\
    =&\ \mathbb{E}_{Z_0^{-1} \sim D_{S_0^{-1}}} \left[\frac{g(z_0^{-1}) - f(z_0^{-1})}{2}\right] + \mathbb{E}_{Z_1^{-1} \sim D_{S_1^{-1}}} \left[\frac{g(z_1^{-1}) + f(z_1^{-1})}{2}\right] - 1  \label{eq:thm_eop_rademacher_pf1}\\
    =&\ \e_{S_0^{-1}}(g,f) + \e_{S_1^{-1}}(-g,f) -1 \nonumber,
\end{align}
where \ref{eq:thm_eop_rademacher_pf1} is due to the fact that $f(z_0^{-1})=-1$ by definition.

We now have the tools to find an upper bound on $\eopdiff{T}(g)$.
\begin{align}
    \eopdiff{T}(g) =&\ \e_{T_0^{-1}}(g,f) + \e_{T_1^{-1}}(-g,f) -1 \nonumber \\
    \leq&\ \e_{T_0^{-1}}(g,g^*) + \e_{T_0^{-1}}(f, g^*) + \e_{T_1^{-1}}(-g,g^*) + \e_{T_1^{-1}}(f,g^*) -1 \label{eq:thm_eop_rademacher_pf2} \\
    =&\ \e_{T_0^{-1}}(f, g^*) + \e_{T_0^{-1}}(g,g^*) + \e_{S_0^{-1}}(g,g^*) - \e_{S_0^{-1}}(g,g^*) \nonumber \\
    &\ + \e_{T_1^{-1}}(f,g^*) + \e_{T_1^{-1}}(-g,g^*) + \e_{S_1^{-1}}(-g,g^*) - \e_{S_1^{-1}}(-g,g^*) - 1 \nonumber \\
    \leq&\ \e_{T_0^{-1}}(g^*, f) + \e_{S_0^{-1}}(g,g^*) + |\e_{T_0^{-1}}(g,g^*) - \e_{S_0^{-1}}(g,g^*)| \nonumber \\
    &\ + \e_{T_1^{-1}}(g^*, f) + \e_{S_1^{-1}}(-g,g^*) + |\e_{T_1^{-1}}(-g,g^*) - \e_{S_1^{-1}}(-g,g^*)| - 1 \nonumber \\
    \leq&\ \e_{T_0^{-1}}(g^*, f) + \e_{S_0^{-1}}(g,g^*) + \frac{1}{2}d_{\symhyp}(D_{T_0^{-1}}, D_{S_0^{-1}}) \nonumber \\
    &\ + \e_{T_1^{-1}}(g^*,f) + \e_{S_1^{-1}}(-g,g^*) + \frac{1}{2}d_{\symhyp}(D_{T_1^{-1}}, D_{S_1^{-1}}) - 1 \label{eq:thm_eop_rademacher_pf3} \\
    \leq&\ \e_{T_0^{-1}}(g^*,f) + \e_{S_0^{-1}}(g,f) + \e_{S_0^{-1}}(g^*,f) + \frac{1}{2}d_{\symhyp}(D_{T_0^{-1}}, D_{S_0^{-1}}) \nonumber \\
    &\ + \e_{T_1^{-1}}(g^*,f) + \e_{S_1^{-1}}(-g,f) + \e_{S_1^{-1}}(g^*,f) + \frac{1}{2}d_{\symhyp}(D_{T_1^{-1}}, D_{S_1^{-1}}) - 1 \label{eq:thm_eop_rademacher_pf4} \\
    =&\ \e_{S_0^{-1}}(g,f) + \e_{S_1^{-1}}(-g,f) - 1 + \frac{1}{2}d_{\symhyp}(D_{T_0^{-1}}, D_{S_0^{-1}}) \nonumber \\
    &\ + \frac{1}{2}d_{\symhyp}(D_{T_1^{-1}}, D_{S_1^{-1}}) + \lambda_0^{-1} + \lambda_1^{-1} \nonumber   \\
    =&\ \eopdiff{S}(g) + \frac{1}{2}d_{\symhyp}(D_{T_0^{-1}}, D_{S_0^{-1}}) + \frac{1}{2}d_{\symhyp}(D_{T_1^{-1}}, D_{S_1^{-1}}) + \lambda_0^{-1} + \lambda_1^{-1} \label{eq:thm_eop_rademacher_pf5}  \\
    \leq&\ \eopdiff{S}(g) + \lambda_0^{-1} + \lambda_1^{-1}  \nonumber \\
    &\ + \frac{1}{2}\left(\hat{d}_{\symhyp}(D_{T_0^{-1}}, D_{S_0^{-1}}) + 4\left(\hat{\mathfrak{R}}_{U_{T_0^{-1}}}(\mathcal{H}) + \hat{\mathfrak{R}}_{U_{S_0^{-1}}}(\mathcal{H})\right) + 6\sqrt{\frac{\log\frac{2}{\delta}}{2m}} \right) \nonumber \\
    &\ + \frac{1}{2}\left(\hat{d}_{\symhyp}(D_{T_1^{-1}}, D_{S_1^{-1}}) + 4\left( \hat{\mathfrak{R}}_{U_{T_0^{-1}}}(\mathcal{H}) + \hat{\mathfrak{R}}_{U_{S_0^{-1}}}(\mathcal{H}) \right) + 6\sqrt{\frac{\log\frac{2}{\delta}}{2m}} \right) \label{eq:thm_eop_rademacher_pf6}  \\
    =&\ \eopdiff{S}(g) + \frac{1}{2}\hat{d}_{\symhyp}(D_{T_0^{-1}}, D_{S_0^{-1}}) + \frac{1}{2}\hat{d}_{\symhyp}(D_{T_1^{-1}}, D_{S_1^{-1}}) \nonumber \\
    &\ + 2\left(\hat{\mathfrak{R}}_{U_{T_0^{-1}}}(\mathcal{H}) + \hat{\mathfrak{R}}_{U_{S_0^{-1}}}(\mathcal{H}) + \hat{\mathfrak{R}}_{U_{T_0^{-1}}}(\mathcal{H}) + \hat{\mathfrak{R}}_{U_{S_0^{-1}}}(\mathcal{H}) \right) \nonumber \\
    &\ + 6\sqrt{\frac{\log\frac{2}{\delta}}{2m}} + \lambda_0^{-1} + \lambda_1^{-1} \nonumber,
\end{align}
where Eq. \ref{eq:thm_eop_rademacher_pf2} is due to Lemma \ref{lem:triangle}, Eq. \ref{eq:thm_eop_rademacher_pf3} is due to Lemma \ref{lem:sym_hyp}, Eq. \ref{eq:thm_eop_rademacher_pf4} is due to Lemma \ref{lem:triangle}, Eq. \ref{eq:thm_eop_rademacher_pf5} is due to the definition of $\eopdiff{S}(g)$, and Eq. \ref{eq:thm_eop_rademacher_pf6} is due to Lemma~\ref{lem:rademacher_dist}.
\end{proof}

\begin{theorem}
Let $\mathcal{H}$ be a hypothesis space. If $\ \mathcal{U}_{S_0^{-1}},\ \mathcal{U}_{S_1^{-1}},\ \mathcal{U}_{T_0^{-1}},\  \mathcal{U}_{T_1^{-1}} \ \mathcal{U}_{S_0^{1}},\ \mathcal{U}_{S_1^{1}},\ \mathcal{U}_{T_0^{1}},\  \mathcal{U}_{T_1^{1}}$ are samples of size $m'$ each, drawn from $\mathcal{D}_{S_0^{-1}}$, $\mathcal{D}_{S_1^{-1}}$, $\mathcal{D}_{T_0^{-1}}$, $\mathcal{D}_{T_1^{-1}}, \mathcal{D}_{S_0^{1}}$, $\mathcal{D}_{S_1^{1}}$, $\mathcal{D}_{T_0^{1}}$, and $\mathcal{D}_{T_1^{1}}$ respectively, then for any $\delta\in(0,1)$, with probability at least $1-\delta$ (over the choice of samples), for every $g\in\mathcal{H}$ (where $\mathcal{H}$ is a symmetric hypothesis space) the distance from equalized odds in the target space is bounded by
\begin{align*}
    \eodiff{T}(g) \leq&\ \eodiff{S}(g) + \frac{1}{2}\left( \hat{d}_{\symhyp}(\mathcal{U}_{S_0^{-1}}, \mathcal{U}_{T_0^{-1}}) + \hat{d}_{\symhyp}(\mathcal{U}_{S_1^{-1}}, \mathcal{U}_{T_1^{-1}}) \right. \\
    &\ \left. + \hat{d}_{\symhyp}(\mathcal{U}_{S_0^{1}}, \mathcal{U}_{T_0^{1}}) + \hat{d}_{\symhyp}(\mathcal{U}_{S_1^{1}}, \mathcal{U}_{T_1^{1}}) \right) \\
    &\ + 2\left( \hat{\mathfrak{R}}_{U_{S_0^{-1}}}(\mathcal{H}) + \hat{\mathfrak{R}}_{U_{T_0^{-1}}}(\mathcal{H}) \right. + \hat{\mathfrak{R}}_{U_{S_1^{-1}}}(\mathcal{H}) + \hat{\mathfrak{R}}_{U_{T_1^{-1}}}(\mathcal{H}) \\
    &\ + \hat{\mathfrak{R}}_{U_{S_0^{1}}}(\mathcal{H}) + \hat{\mathfrak{R}}_{U_{T_0^{1}}}(\mathcal{H}) \left. + \hat{\mathfrak{R}}_{U_{S_1^{1}}}(\mathcal{H}) + \hat{\mathfrak{R}}_{U_{T_1^{1}}}(\mathcal{H}) \right) \\
    &\ + 12\sqrt{\frac{\log\frac{2}{\delta}}{2m}} + \lambda_\mathit{EO},
\end{align*}
where $\lambda_\mathit{EO} = \lambda_0^{-1} + \lambda_1^{-1} + \lambda_0^1 + \lambda_1^1$, and $\lambda_\alpha^l=\e_{S_\alpha^l}(g^*,f)+\e_{T_\alpha^l}(g^*,f)$.
\end{theorem}

\begin{proof}
Without loss of generality assume $\mathbb{E}_{Z_0^{-1} \sim D_{S_0^{-1}}} \geq \mathbb{E}_{Z_1^{-1} \sim D_{S_1^{-1}}}$ and $\mathbb{E}_{Z_0^{1} \sim D_{S_0^{1}}} \geq \mathbb{E}_{Z_1^{1} \sim D_{S_1^{1}}}$. Then we can rewrite $\eopdiff{S}$ as follows.
\begin{align*}
    \eodiff{T}(g) =&\ \mathbb{E}_{Z_0^{-1}\sim D_{T_0^{-1}}}\left[\frac{1+g(z_0^{-1})}{2}\right] - \mathbb{E}_{Z_1^{-1}\sim D_{T_1^{-1}}}\left[\frac{1+g(z_1^{-1})}{2}\right] \\
        &\ + \mathbb{E}_{Z_0^{1}\sim D_{T_0^{1}}}\left[\frac{1+g(z_0^{1})}{2}\right] - \mathbb{E}_{Z_1^{1}\sim D_{T_1^{1}}}\left[\frac{1+g(z_1^{1})}{2}\right] \\
    =&\ \mathbb{E}_{Z_0^{-1}\sim D_{T_0^{-1}}}\left[\frac{1+g(z_0^{-1})}{2}\right] + \mathbb{E}_{Z_1^{-1}\sim D_{T_1^{-1}}}\left[1 - \frac{1+g(z_1^{-1})}{2}\right] - 1 \\
        &\ + \mathbb{E}_{Z_0^{1}\sim D_{T_0^{1}}}\left[\frac{1+g(z_0^{1})}{2}\right] + \mathbb{E}_{Z_1^{1}\sim D_{T_1^{1}}}\left[1 - \frac{1+g(z_1^{1})}{2}\right] - 1 \\
    =&\ \mathbb{E}_{Z_0^{-1}\sim D_{T_0^{-1}}}\left[\frac{1+g(z_0^{-1})}{2}\right] + \mathbb{E}_{Z_1^{-1}\sim D_{T_1^{-1}}}\left[ \frac{1-g(z_1^{-1})}{2}\right]  \\
        &\ + \mathbb{E}_{Z_0^{1}\sim D_{T_0^{1}}}\left[\frac{1+g(z_0^{1})}{2}\right] + \mathbb{E}_{Z_1^{1}\sim D_{T_1^{1}}}\left[\frac{1-g(z_1^{1})}{2}\right] - 2 \\
    =&\ \mathbb{E}_{Z_0^{-1}\sim D_{T_0^{-1}}}\left[\frac{|g(z_0^{-1}) - f(z_0^{-1})|}{2}\right] + \mathbb{E}_{Z_1^{-1}\sim D_{T_1^{-1}}}\left[ \frac{|g(z_1^{-1}) + f(z_1^{-1})|}{2}\right]  \\
        &\ + \mathbb{E}_{Z_0^{1}\sim D_{T_0^{1}}}\left[\frac{|g(z_0^{1}) - f(z_0^{1})|}{2}\right]  + \mathbb{E}_{Z_1^{1}\sim D_{T_1^{1}}}\left[\frac{|g(z_1^{1}) + f(z_1^{1})|}{2}\right] - 2 \\
    =&\ \e_{T_0^{-1}}(g,f) + \e_{T_1^{-1}}(-g,f) + \e_{T_0^1}(g,f) + \e_{T_1^1}(-g,f) - 2
\end{align*}

Using this and previous lemmas we have
\begin{align}
    \eodiff{T}(g) =&\ \e_{T_0^{-1}}(g,f) + \e_{T_1^{-1}}(-g,f) + \e_{T_0^1}(g,f) + \e_{T_1^1}(-g,f) - 2 \nonumber \\
    \leq&\ \e_{T_0^{-1}}(g,g^*) + \e_{T_0^{-1}}(f,g^*) + \e_{T_1^{-1}}(-g,g^*) + \e_{T_1^{-1}}(f,g^*) \nonumber \\
        &\ + \e_{T_0^1}(g,g^*) + \e_{T_0^1}(f,g^*) + \e_{T_1^1}(-g,g^*) + \e_{T_1^1}(f,g^*) - 2 \label{eq:eo_rademacher_1} \\
    =&\ \e_{T_0^{-1}}(f,g^*) + \e_{T_0^{-1}}(g,g^*) + \e_{S_0^{-1}}(g,g^*) - \e_{S_0^{-1}}(g,g^*) \nonumber \\
        &\ + \e_{T_1^{-1}}(f,g^*) + \e_{T_1^{-1}}(-g,g^*) + \e_{S_1^{-1}}(-g,g^*) - \e_{S_1^{-1}}(-g,g^*) \nonumber \\
        &\ + \e_{T_0^1}(f,g^*) + \e_{T_0^1}(g,g^*) + \e_{S_0^1}(g,g^*) - \e_{S_0^1}(g,g^*) \nonumber \\
        &\ + \e_{T_1^1}(f,g^*) + \e_{T_1^1}(-g,g^*) + \e_{S_1^1}(-g,g^*) - \e_{S_1^1}(-g,g^*) - 2 \nonumber \\
    \leq&\ \e_{T_0^{-1}}(f,g^*) + \e_{S_0^{-1}}(g,g^*) + \left| \e_{T_0^{-1}}(g,g^*) - \e_{S_0^{-1}}(g,g^*) \right| \nonumber \\
        &\ + \e_{T_1^{-1}}(f,g^*) + \e_{S_1^{-1}}(-g,g^*) + \left| \e_{T_1^{-1}}(-g,g^*) - \e_{S_1^{-1}}(-g,g^*) \right| \nonumber \\
        &\ + \e_{T_0^1}(f,g^*) + \e_{S_0^1}(g,g^*) + \left| \e_{T_0^1}(g,g^*) - \e_{S_0^1}(g,g^*) \right| \nonumber \\
        &\ + \e_{T_1^1}(f,g^*) + \e_{S_1^1}(-g,g^*) + \left| \e_{T_1^1}(-g,g^*) - \e_{S_1^1}(-g,g^*) \right| - 2 \nonumber \\
    \leq&\ \e_{T_0^{-1}}(f,g^*) + \e_{S_0^{-1}}(g,g^*) + \frac{1}{2}d_{\symhyp}(D_{T_0^{-1}}, D_{S_0^{-1}}) \nonumber \\
        &\ + \e_{T_1^{-1}}(f,g^*) + \e_{S_1^{-1}}(-g,g^*) + \frac{1}{2}d_{\symhyp}(D_{T_1^{-1}}, D_{S_1^{-1}}) \nonumber \\
        &\ + \e_{T_0^1}(f,g^*) + \e_{S_0^1}(g,g^*) + \frac{1}{2}d_{\symhyp}(D_{T_0^{1}}, D_{S_0^{1}}) \nonumber \\
        &\ + \e_{T_1^1}(f,g^*) + \e_{S_1^1}(-g,g^*) + \frac{1}{2}d_{\symhyp}(D_{T_1^{1}}, D_{S_1^{1}}) - 2 \label{eq:eo_rademacher_2} \\
    \leq&\ \e_{T_0^{-1}}(f,g^*) + \e_{S_0^{-1}}(g,f) + \e_{S_0^{-1}}(g^*,f) + \frac{1}{2}d_{\symhyp}(D_{T_0^{-1}}, D_{S_0^{-1}}) \nonumber \\
        &\ + \e_{T_1^{-1}}(f,g^*) + \e_{S_1^{-1}}(-g,f) + \e_{S_1^{-1}}(g^*,f) + \frac{1}{2}d_{\symhyp}(D_{T_1^{-1}}, D_{S_1^{-1}}) \nonumber \\
        &\ + \e_{T_0^1}(f,g^*) + \e_{S_0^1}(g,f) + \e_{S_0^1}(g^*,f) + \frac{1}{2}d_{\symhyp}(D_{T_0^{1}}, D_{S_0^{1}}) \nonumber \\
        &\ + \e_{T_1^1}(f,g^*) + \e_{S_1^1}(-g,f) + \e_{S_1^1}(g^*,f) + \frac{1}{2}d_{\symhyp}(D_{T_1^{1}}, D_{S_1^{1}}) - 2 \label{eq:eo_rademacher_3} \\
    =&\ \e_{S_0^{-1}}(g,f) + \e_{S_1^{-1}}(-g,f) + \e_{S_0^1}(g,f) + \e_{S_1^1}(-g,f) - 2 \nonumber \\
        &\ + \frac{1}{2}d_{\symhyp}(D_{T_0^{-1}}, D_{S_0^{-1}}) + \frac{1}{2}d_{\symhyp}(D_{T_1^{-1}}, D_{S_1^{-1}}) \nonumber \\
        &\ + \frac{1}{2}d_{\symhyp}(D_{T_0^{1}}, D_{S_0^{1}}) + \frac{1}{2}d_{\symhyp}(D_{T_1^{1}}, D_{S_1^{1}}) + \lambda_0^{-1} + \lambda_1^{-1} + \lambda_0^1 + \lambda_1^1 \nonumber \\
    =&\ \eodiff{S}(g) + \frac{1}{2}d_{\symhyp}(D_{T_0^{-1}}, D_{S_0^{-1}}) + \frac{1}{2}d_{\symhyp}(D_{T_1^{-1}}, D_{S_1^{-1}}) \nonumber \\
        &\ + \frac{1}{2}d_{\symhyp}(D_{T_0^{1}}, D_{S_0^{1}}) + \frac{1}{2}d_{\symhyp}(D_{T_1^{1}}, D_{S_1^{1}}) + \lambda_0^{-1} + \lambda_1^{-1} + \lambda_0^1 + \lambda_1^1 \nonumber \\
    \leq&\ \eodiff{S}(g) + \lambda_0^{-1} + \lambda_1^{-1} + \lambda_0^1 + \lambda_1^1 \nonumber \\
        &\ + \frac{1}{2}\left( \hat{d}_{\symhyp}(D_{T_0^{-1}}, D_{S_0^{-1}}) + 4\left( \hat{\mathfrak{R}}_{U_{S_0^{-1}}}(\mathcal{H}) + \hat{\mathfrak{R}}_{U_{T_0^{-1}}}(\mathcal{H}) \right) + 6\sqrt{\frac{\log\frac{2}{\delta}}{2m}} \right) \nonumber \\
        &\ + \frac{1}{2}\left( \hat{d}_{\symhyp}(D_{T_1^{-1}}, D_{S_1^{-1}}) + 4\left( \hat{\mathfrak{R}}_{U_{S_1^{-1}}}(\mathcal{H}) + \hat{\mathfrak{R}}_{U_{T_1^{-1}}}(\mathcal{H}) \right) + 6\sqrt{\frac{\log\frac{2}{\delta}}{2m}} \right) \nonumber \\
        &\ + \frac{1}{2}\left( \hat{d}_{\symhyp}(D_{T_0^{1}}, D_{S_0^{1}})  + 4\left( \hat{\mathfrak{R}}_{U_{S_0^{1}}}(\mathcal{H}) + \hat{\mathfrak{R}}_{U_{T_0^{1}}}(\mathcal{H}) \right) + 6\sqrt{\frac{\log\frac{2}{\delta}}{2m}} \right) \nonumber \\
        &\ + \frac{1}{2}\left( \hat{d}_{\symhyp}(D_{T_1^{1}}, D_{S_1^{1}})  + 4\left( \hat{\mathfrak{R}}_{U_{S_1^{1}}}(\mathcal{H}) + \hat{\mathfrak{R}}_{U_{T_1^{1}}}(\mathcal{H}) \right) + 6\sqrt{\frac{\log\frac{2}{\delta}}{2m}} \right) \label{eq:eo_rademacher_4}\\
    =&\ \eodiff{S}(g) + \frac{1}{2}\left( \hat{d}_{\symhyp}(\mathcal{U}_{S_0^{-1}}, \mathcal{U}_{T_0^{-1}}) + \hat{d}_{\symhyp}(\mathcal{U}_{S_1^{-1}}, \mathcal{U}_{T_1^{-1}}) \right. \nonumber \\
        &\ \left. + \hat{d}_{\symhyp}(\mathcal{U}_{S_0^{1}}, \mathcal{U}_{T_0^{1}}) + \hat{d}_{\symhyp}(\mathcal{U}_{S_1^{1}}, \mathcal{U}_{T_1^{1}}) \right) \nonumber \\
        &\ + 2\left( \hat{\mathfrak{R}}_{U_{S_0^{-1}}}(\mathcal{H}) + \hat{\mathfrak{R}}_{U_{T_0^{-1}}}(\mathcal{H}) \right. + \hat{\mathfrak{R}}_{U_{S_1^{-1}}}(\mathcal{H}) + \hat{\mathfrak{R}}_{U_{T_1^{-1}}}(\mathcal{H}) \nonumber \\
        &\ + \hat{\mathfrak{R}}_{U_{S_0^{1}}}(\mathcal{H}) + \hat{\mathfrak{R}}_{U_{T_0^{1}}}(\mathcal{H}) \left. + \hat{\mathfrak{R}}_{U_{S_1^{1}}}(\mathcal{H}) + \hat{\mathfrak{R}}_{U_{T_1^{1}}}(\mathcal{H}) \right) \nonumber \\
        &\ + 12\sqrt{\frac{\log\frac{2}{\delta}}{2m}} + \lambda_0^{-1} + \lambda_1^{-1} + \lambda_0^1 + \lambda_1^1 \nonumber,
\end{align}
where Eq. \ref{eq:eo_rademacher_1} is due to Lemma \ref{lem:triangle}, Eq. \ref{eq:eo_rademacher_2} is due to Lemma \ref{lem:sym_hyp}, Eq. \ref{eq:eo_rademacher_3} is due to Lemma \ref{lem:triangle}, and \ref{eq:eo_rademacher_4} is due to Lemma~\ref{lem:rademacher_dist}.
\end{proof}

\section{Experimental setup}
For the UCI adult dataset we used all 14 features as provided in \url{https://archive.ics.uci.edu/ml/machine-learning-databases/adult/adult.names}. 
The original train/test split is used.
For the COMPAS dataset we used the features provided in \url{https://github.com/propublica/compas-analysis/blob/master/compas-scores.csv}, and predict the risk of recidivism (decile\_score) for each row.

We did 10-fold cross-validation and choose the hyperparameters with the best performance on the validation data.
$64$ dimension embedding is used for categorical features and $256$ hidden units are used in the model. We did parameter search and found $10$K steps yields a good balance of runtime and accuracy. Each run takes about 1hr for UCI data and 0.5hrs for COMPAS on a single CPU with 2GB RAM. Increasing learning rate speeds up experiments but also hurts accuracy slightly (e.g., \textasciitilde2pp decrease on UCI).

For range of parameters, we have considered the following:
(1) batch size: $[64, 128, 256, 512]$; (2) learning rate: $[0.01, 0.1, 1.0]$; (3) number of hidden units: $[64, 128, 256, 512]$; (4) embedding dimension: $[32, 64, 128]$. (5) number of steps: $[5000, 10000, 20000, 50000]$. 

\section{Experiments}
\label{appendix_exp}
\subsection{Experiment Results for fairness on UCI and COMPAS}
Figure~\ref{fig:fpr_diff_race_uci} depicts the results of the analysis for
transferring from gender to race, while Figure~\ref{fig:fpr_diff_gender_uci} shows the
results for transferring from race to gender, on the UCI dataset. 
Figure~\ref{fig:fpr_diff_race_compas} and Figure~\ref{fig:fpr_diff_gender_compas} show the results on the COMPAS dataset. 
The line and the shaded areas show the mean and the standard error of the mean across 30 trials. 
These experiments show that the Transfer model is effective in decreasing the FPR gap in the target domain and is more sample efficient than previous methods.
\begin{figure}
    \centering
    \begin{subfigure}{0.23\columnwidth}
    \includegraphics[width=\textwidth]{plots/uci/mmd_race_compare_race_fpr_diff_gender1000_race50.png}
    \caption{$50$ race samples.}
    \label{fig:fpr_diff_race_50}
    \end{subfigure}
    \hspace{0.01\columnwidth}
    \begin{subfigure}{0.23\columnwidth}
    \includegraphics[width=\textwidth]{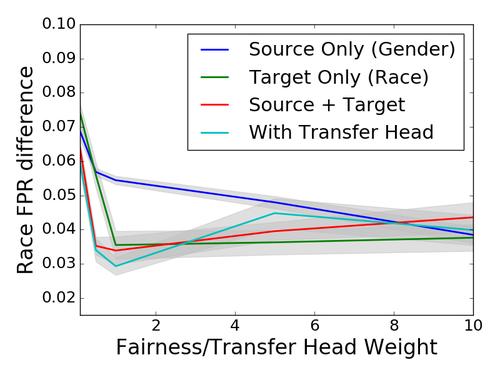}
    \caption{$100$ race samples.}
    \label{fig:fpr_diff_race_100}
    \end{subfigure}
    \hspace{0.01\columnwidth}
    \begin{subfigure}{0.23\columnwidth}
    \includegraphics[width=\textwidth]{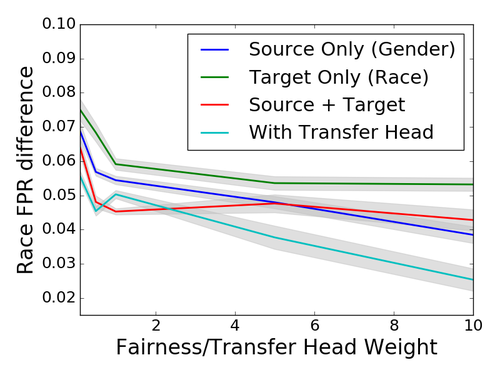}
    \caption{$500$ race samples.}
    \end{subfigure}
    \hspace{0.01\columnwidth}
    \begin{subfigure}{0.23\columnwidth}
    \includegraphics[width=\textwidth]{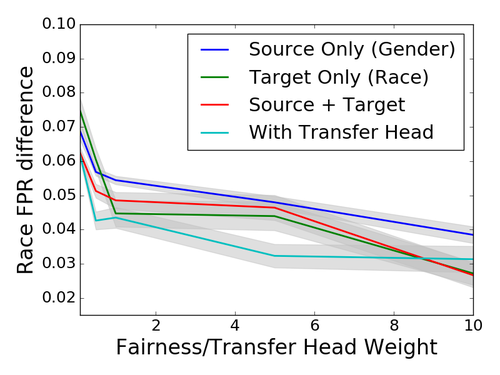}
    \caption{$1000$ race samples.}
    \end{subfigure}
    \caption{Gender $\rightarrow$ Race on the UCI dataset. Comparison of FPR difference on sensitive attribute \textit{race}, by transferring from the source domain (1000 samples for each gender) to the target domain (varying samples for each race as indicated in the caption).}
    \label{fig:fpr_diff_race_uci}
\end{figure}

\begin{figure}
    \centering
    \begin{subfigure}{0.23\columnwidth}
    \includegraphics[width=\textwidth]{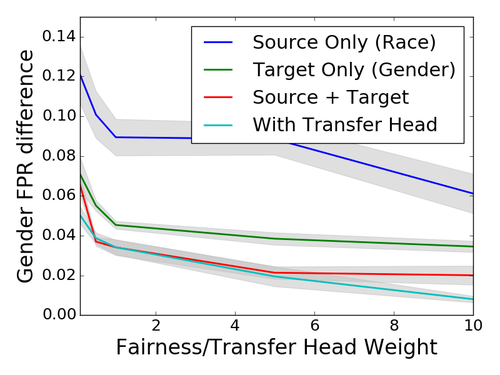}
    \caption{$50$ gender samples.}
    \label{fig:fpr_diff_gender_50}
    \end{subfigure}
    \hspace{0.01\columnwidth}
    \begin{subfigure}{0.23\columnwidth}
    \includegraphics[width=\textwidth]{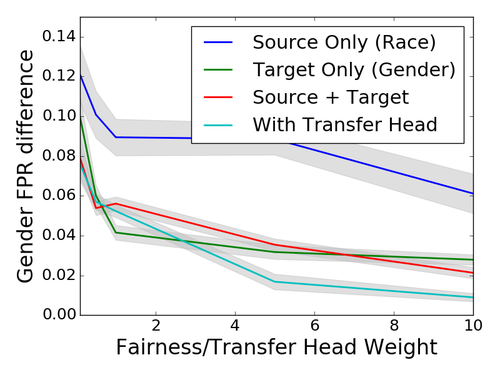}
    \caption{$100$ gender samples.}
    \label{fig:fpr_diff_gender_100}
    \end{subfigure}
    \hspace{0.01\columnwidth}
    \begin{subfigure}{0.23\columnwidth}
    \includegraphics[width=\textwidth]{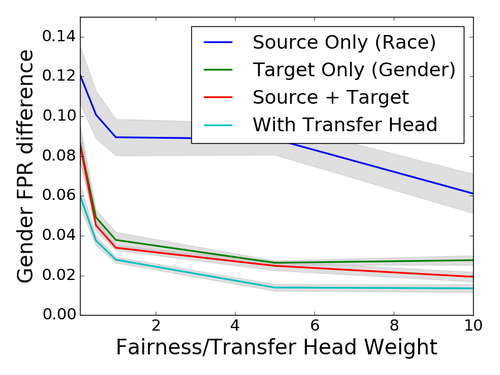}
    \caption{$500$ gender samples.}
    \end{subfigure}
    \hspace{0.01\columnwidth}
    \begin{subfigure}{0.23\columnwidth}
    \includegraphics[width=\textwidth]{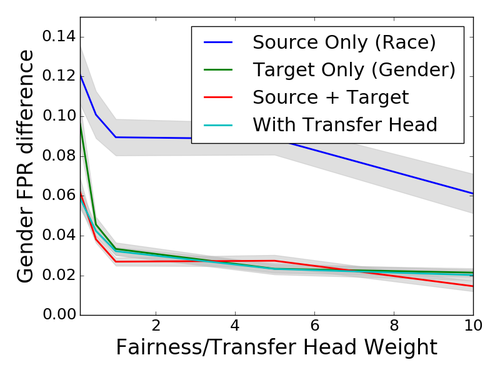}
    \caption{$1000$ gender samples.}
    \end{subfigure}
    \caption{Race $\rightarrow$ Gender on the UCI dataset. Comparison of FPR difference on sensitive attribute \textit{gender}, by transferring from the source domain (1000 samples for each gender) to the target domain (varying samples for each race as indicated in the caption).}
    \label{fig:fpr_diff_gender_uci}
\end{figure}

\begin{figure}
    \centering
    \begin{subfigure}{0.23\columnwidth}
    \includegraphics[width=\textwidth]{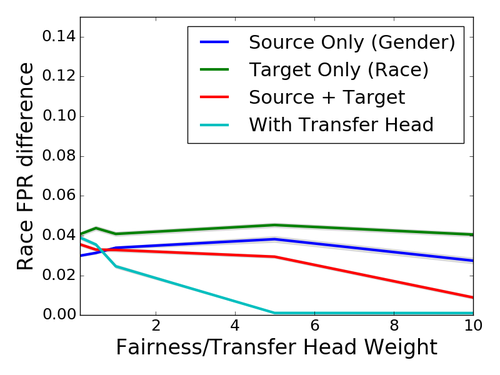}
    \caption{$50$ race samples.}
    \label{fig:fpr_diff_race_50_compas}
    \end{subfigure}
    \hspace{0.01\columnwidth}
    \begin{subfigure}{0.23\columnwidth}
    \includegraphics[width=\textwidth]{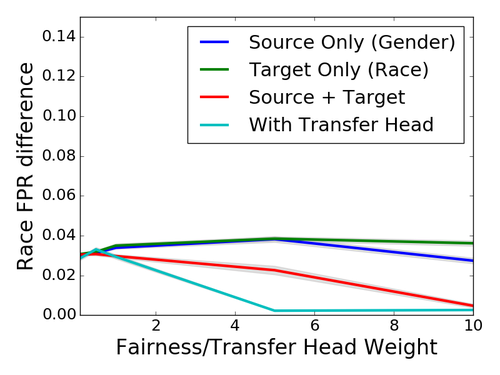}
    \caption{$100$ race samples.}
    \label{fig:fpr_diff_race_100_compas}
    \end{subfigure}
    \hspace{0.01\columnwidth}
    \begin{subfigure}{0.23\columnwidth}
    \includegraphics[width=\textwidth]{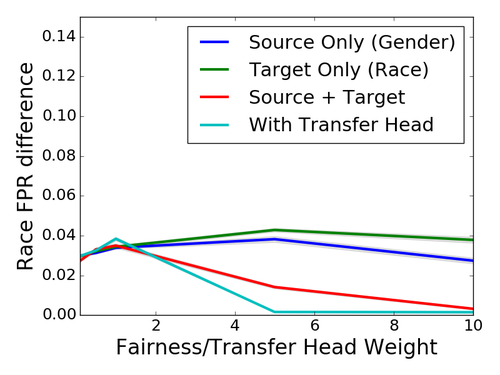}
    \caption{$500$ race samples.}
    \end{subfigure}
    \hspace{0.01\columnwidth}
    \begin{subfigure}{0.23\columnwidth}
    \includegraphics[width=\textwidth]{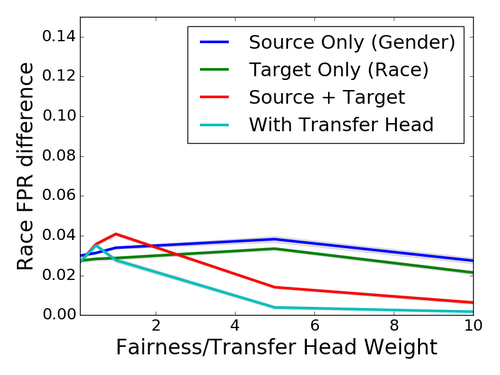}
    \caption{$1000$ race samples.}
    \end{subfigure}
    \caption{Gender $\rightarrow$ Race on the COMPAS dataset. Comparison of FPR difference on sensitive attribute \textit{race}, by transferring from the source domain (1000 samples for each gender) to the target domain (varying samples for each race as indicated in the caption).}
    \label{fig:fpr_diff_race_compas}
\end{figure}

\begin{figure}
    \centering
    \begin{subfigure}{0.23\columnwidth}
    \includegraphics[width=\textwidth]{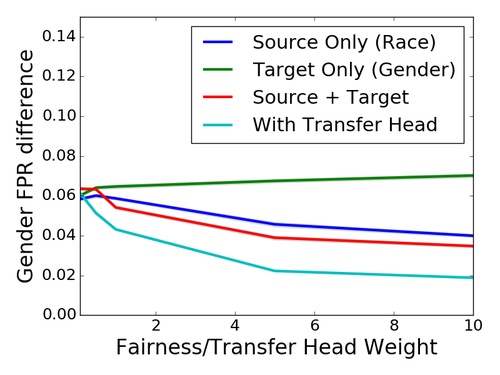}
    \caption{$50$ gender samples.}
    \label{fig:fpr_diff_gender_50_compas}
    \end{subfigure}
    \hspace{0.01\columnwidth}
    \begin{subfigure}{0.23\columnwidth}
    \includegraphics[width=\textwidth]{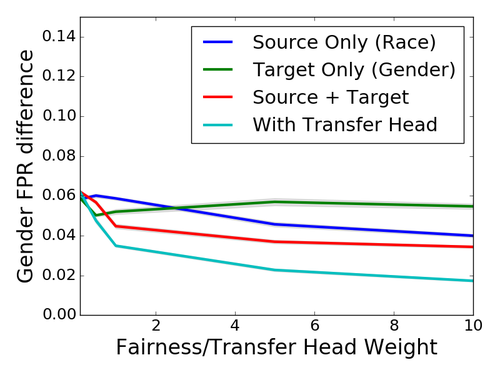}
    \caption{$100$ gender samples.}
    \label{fig:fpr_diff_gender_100_compas}
    \end{subfigure}
    \hspace{0.01\columnwidth}
    \begin{subfigure}{0.23\columnwidth}
    \includegraphics[width=\textwidth]{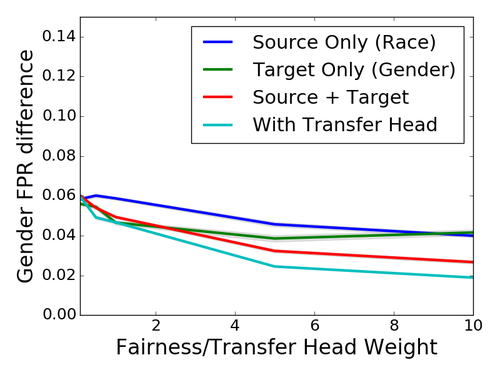}
    \caption{$500$ gender samples.}
    \end{subfigure}
    \hspace{0.01\columnwidth}
    \begin{subfigure}{0.23\columnwidth}
    \includegraphics[width=\textwidth]{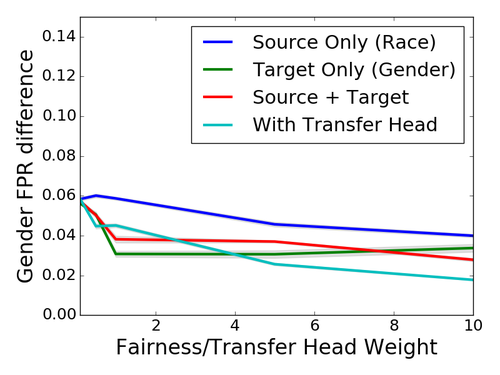}
    \caption{$1000$ gender samples.}
    \end{subfigure}
    \caption{Race $\rightarrow$ Gender on the COMPAS dataset. Comparison of FPR difference on sensitive attribute \textit{gender}, by transferring from the source domain (1000 samples for each gender) to the target domain (varying samples for each race as indicated in the caption).}
    \label{fig:fpr_diff_gender_compas}
    \vspace{-0.1in}
\end{figure}

\subsection{Accuracy vs. Fairness/Transfer Head Weight}
In this section we further add the comparison on accuracy with respect to the weight of the fairness/transfer head.
Fig.~\ref{fig:accuracy_race_to_gender_uci} and Fig.~\ref{fig:accuracy_gender_to_race_uci} show the results comparing the Transfer model with the baselines, by transferring \textit{race} to \textit{gender}, and \textit{race} to \textit{gender}, respectively.
Fig.~\ref{fig:accuracy_race_to_gender_compas} and Fig.~\ref{fig:accuracy_gender_to_race_compas} show the results on COMPAS.

\begin{figure}[h]
    \centering
    \begin{subfigure}{0.23\columnwidth}
    \includegraphics[width=\textwidth]{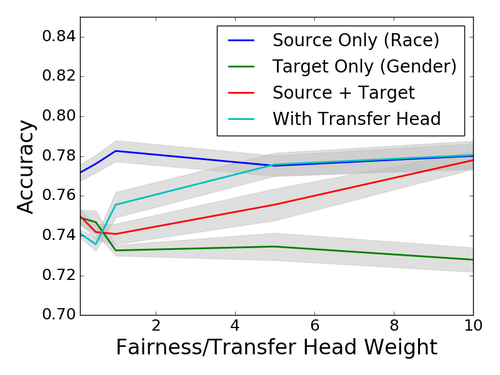}
    \caption{$50$ gender samples.}
    \end{subfigure}
    \hspace{0.01\columnwidth}
    \begin{subfigure}{0.23\columnwidth}
    \includegraphics[width=\textwidth]{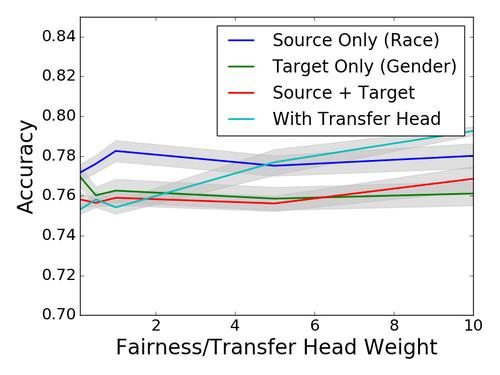}
    \caption{$100$ gender samples.}
    \end{subfigure}
    \hspace{0.01\columnwidth}
    \begin{subfigure}{0.23\columnwidth}
    \includegraphics[width=\textwidth]{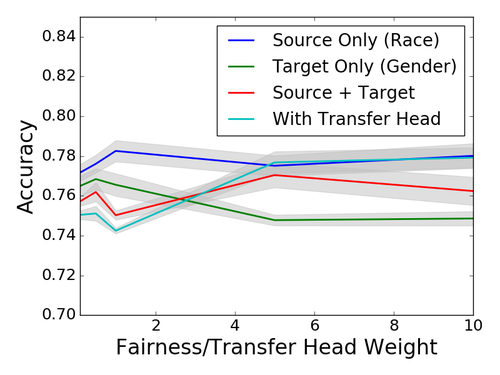}
    \caption{$500$ gender samples.}
    \end{subfigure}
    \hspace{0.01\columnwidth}
    \begin{subfigure}{0.23\columnwidth}
    \includegraphics[width=\textwidth]{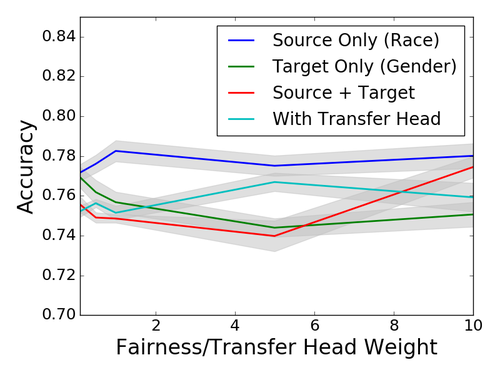}
    \caption{$1000$ gender samples.}
    \end{subfigure}
    \caption{Comparison of accuracy on the UCI data for Race $\rightarrow$ Gender, by transferring from the source domain (1000 samples for each race) to the target domain (varying samples for each gender as indicated in the caption).}
    \label{fig:accuracy_race_to_gender_uci}
    \vspace{-0.2in}
\end{figure}

\begin{figure}[h]
    \centering
    \begin{subfigure}{0.23\columnwidth}
    \includegraphics[width=\textwidth]{plots/uci/mmd_race_compare_race_fpr_diff_gender1000_race50_acc_weight.png}
    \caption{$50$ race samples.}
    \end{subfigure}
    \hspace{0.01\columnwidth}
    \begin{subfigure}{0.23\columnwidth}
    \includegraphics[width=\textwidth]{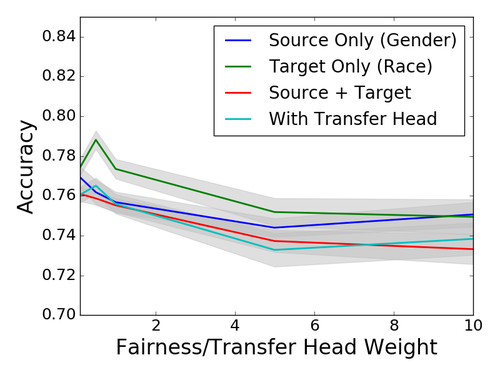}
    \caption{$100$ race samples.}
    \end{subfigure}
    \hspace{0.01\columnwidth}
    \begin{subfigure}{0.23\columnwidth}
    \includegraphics[width=\textwidth]{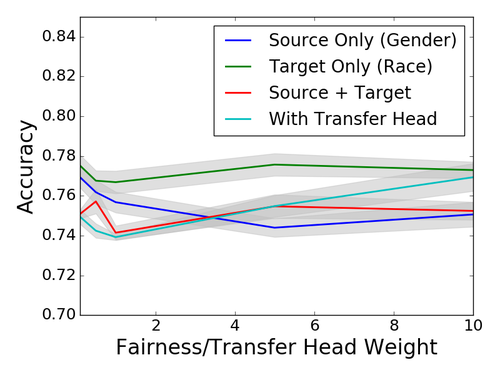}
    \caption{$500$ race samples.}
    \end{subfigure}
    \hspace{0.01\columnwidth}
    \begin{subfigure}{0.23\columnwidth}
    \includegraphics[width=\textwidth]{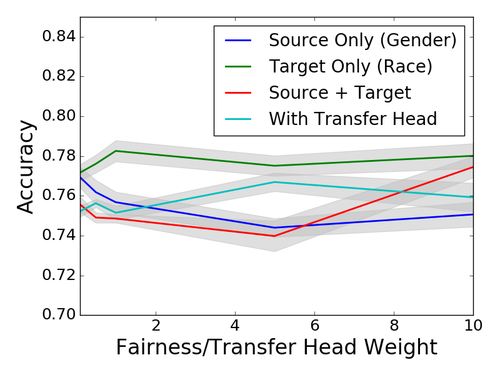}
    \caption{$1000$ race samples.}
    \end{subfigure}
    \caption{Comparison of accuracy on the UCI data for Gender $\rightarrow$ Race, by transferring from the source domain (1000 samples for each gender) to the target domain (varying samples for each race as indicated in the caption).}
    \label{fig:accuracy_gender_to_race_uci}
    \vspace{-0.2in}
\end{figure}

\begin{figure}[h]
    \centering
    \begin{subfigure}{0.23\columnwidth}
    \includegraphics[width=\textwidth]{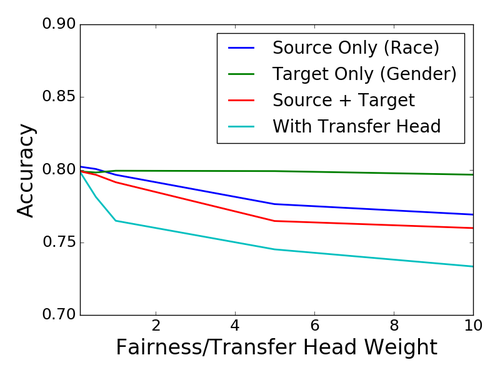}
    \caption{$50$ gender samples.}
    \end{subfigure}
    \hspace{0.01\columnwidth}
    \begin{subfigure}{0.23\columnwidth}
    \includegraphics[width=\textwidth]{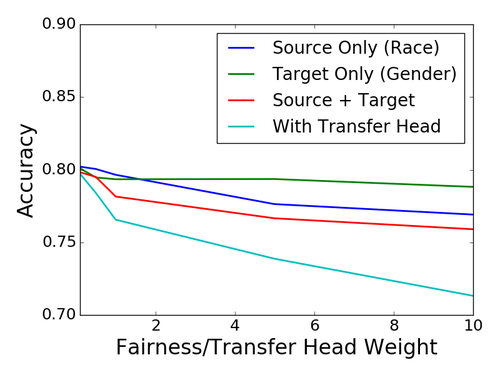}
    \caption{$100$ gender samples.}
    \end{subfigure}
    \hspace{0.01\columnwidth}
    \begin{subfigure}{0.23\columnwidth}
    \includegraphics[width=\textwidth]{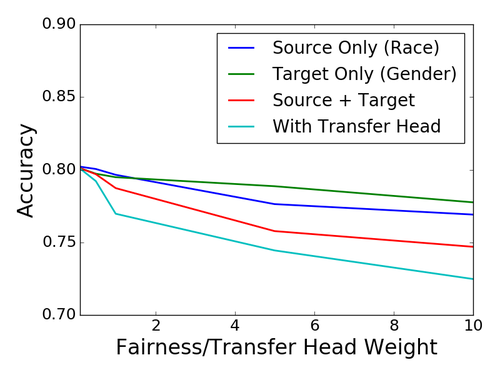}
    \caption{$500$ gender samples.}
    \end{subfigure}
    \hspace{0.01\columnwidth}
    \begin{subfigure}{0.23\columnwidth}
    \includegraphics[width=\textwidth]{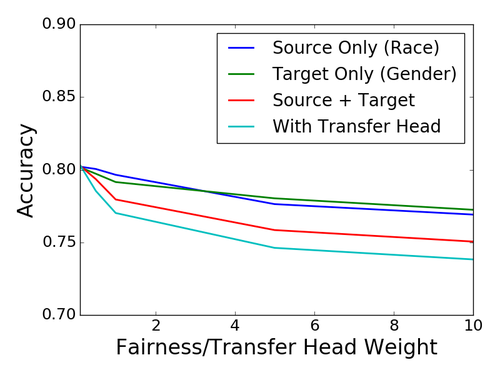}
    \caption{$1000$ gender samples.}
    \end{subfigure}
    \caption{Comparison of accuracy on COMPAS for Race $\rightarrow$ Gender, by transferring from the source domain (1000 samples for each race) to the target domain (varying samples for each gender as indicated in the caption).}
    \label{fig:accuracy_race_to_gender_compas}
    \vspace{-0.2in}
\end{figure}

\begin{figure}[h]
    \centering
    \begin{subfigure}{0.23\columnwidth}
    \includegraphics[width=\textwidth]{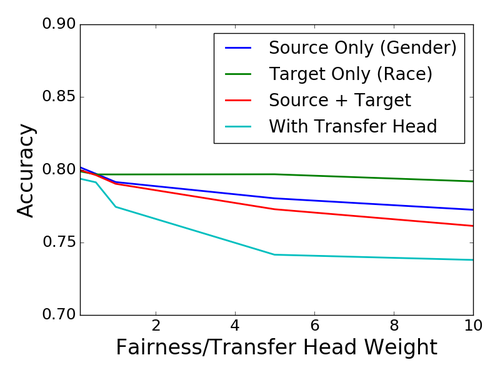}
    \caption{$50$ race samples.}
    \end{subfigure}
    \hspace{0.01\columnwidth}
    \begin{subfigure}{0.23\columnwidth}
    \includegraphics[width=\textwidth]{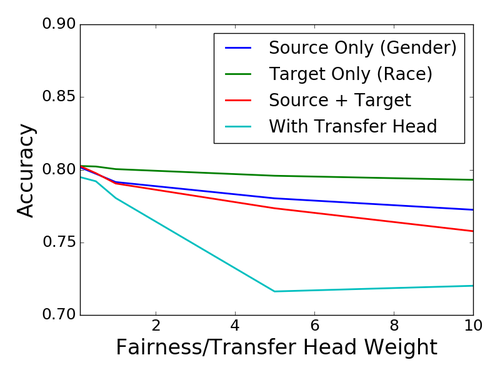}
    \caption{$100$ race samples.}
    \end{subfigure}
    \hspace{0.01\columnwidth}
    \begin{subfigure}{0.23\columnwidth}
    \includegraphics[width=\textwidth]{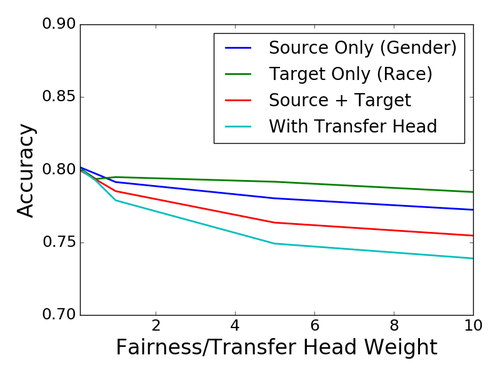}
    \caption{$500$ race samples.}
    \end{subfigure}
    \hspace{0.01\columnwidth}
    \begin{subfigure}{0.23\columnwidth}
    \includegraphics[width=\textwidth]{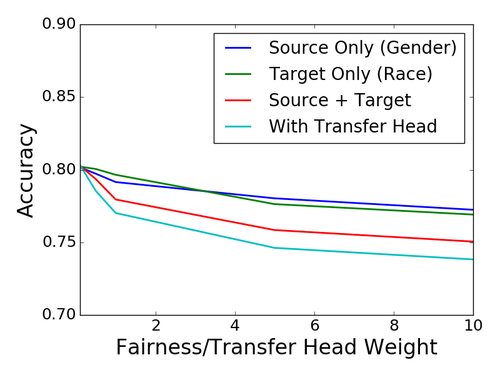}
    \caption{$1000$ race samples.}
    \end{subfigure}
    \caption{Comparison of accuracy on COMPAS for Gender $\rightarrow$ Race, by transferring from the source domain (1000 samples for each gender) to the target domain (varying samples for each race as indicated in the caption).}
    \label{fig:accuracy_gender_to_race_compas}
    \vspace{-0.2in}
\end{figure}

\end{document}